\documentclass{article}

\usepackage[nonatbib,final]{nips}

\usepackage[utf8]{inputenc} 
\usepackage[T1]{fontenc}    
\usepackage{hyperref}       
\usepackage{url}            
\usepackage{booktabs}       
\usepackage{amsfonts}       
\usepackage{nicefrac}       
\usepackage{microtype}      

\usepackage{algorithm}
\usepackage{algorithmic}
\usepackage[numbers]{natbib}
\usepackage{graphicx} 
\usepackage{subfig}
\usepackage{wrapfig}

\usepackage{xspace}
\newcommand{\ourMethod}{\textsc{svdgp}} 
\newcommand{\ourModel}{DGP} 
\usepackage{amsthm}
\usepackage{amssymb}
\usepackage{amsmath}
\usepackage{bm}

\def\BB{\mathcal{B}}
\def\DD{\mathcal{D}}\def\FF{\mathcal{F}}
\def\GG{\mathcal{G}}\def\HH{\mathcal{H}}
\def\LL{\mathcal{L}}
\def\MM{\mathcal{M}}\def\NN{\mathcal{N}}
\def\PP{\mathcal{P}}

\def\XX{\mathcal{X}}
\def\YY{\mathcal{Y}}

\def\Cbb{\mathbb{C}}
\def\Ebb{\mathbb{E}}

\def\Rbb{\mathbb{R}}

\def\R{\Rbb}\def\C{\Cbb}

\def\E{\Ebb}

\def\GP{\GG\PP}
\def\der{\mathrm{d}}

\newcommand{\KL}[2] { \mathrm{KL}[ #1 || #2 ]   }
\newcommand{\norm}[1]{\| #1 \|}
\newcommand{\abs}[1]{|#1|}
\newcommand{\mbf}[1]{\mathbf{#1}}

\def\diag{\mathrm{diag}}

\newcommand{\tr}[1]{ \mathrm{tr}\left( #1\right)}

\newtheorem{theorem}{Theorem}
\newtheorem{proposition}{Proposition}

\def\ifCompileAppendix{\iftrue}

\usepackage{comment}
\usepackage{color}

\title{Variational Inference for Gaussian Process Models with Linear Complexity}

\author{
  Ching-An Cheng \\
 \hspace{-2mm} Institute for Robotics and Intelligent Machines\\
  Georgia Institute of Technology\\
  Atlanta, GA 30332 \\
  \texttt{cacheng@gatech.edu} \\
 \And
 Byron Boots\\
  Institute for Robotics and Intelligent Machines \hspace{-2mm}\\
  Georgia Institute of Technology\\
  Atlanta, GA 30332 \\
  \texttt{bboots@cc.gatech.edu} \\
}

\begin{document}
\maketitle
\setlength{\textfloatsep}{1\baselineskip}
\setlength{\intextsep}{1\baselineskip}

\begin{abstract}
Large-scale Gaussian process inference has long faced practical challenges due to time and space complexity that is superlinear in dataset size. While sparse variational Gaussian process models are capable of learning from large-scale data, standard strategies for sparsifying the model can prevent the approximation of complex functions. In this work, we propose a novel variational Gaussian process model that decouples the representation of mean and covariance functions in reproducing kernel Hilbert space. We show that this new parametrization generalizes previous models. Furthermore, it yields a variational inference problem that can be solved by stochastic gradient ascent with time and space complexity that is only linear in the number of mean function parameters, regardless of the choice of kernels, likelihoods, and inducing points. This strategy makes the adoption of large-scale expressive Gaussian process models possible. We run several experiments on regression tasks and show that this decoupled approach greatly outperforms previous sparse variational Gaussian process inference procedures.  
\end{abstract}

\section{Introduction} \label{sec:introduction}

Gaussian process (GP) inference is a popular nonparametric framework for reasoning about functions under uncertainty.  
However, the expressiveness of GPs comes at a price: 
solving (approximate) inference for a GP with $N$ data instances has time and space complexities  in $\Theta(N^3)$ and $\Theta(N^2)$, respectively.
Therefore, GPs have traditionally been viewed as a tool for problems with small- or medium-sized datasets

Recently, the concept of inducing points has been used to scale GPs to larger datasets. The idea is to summarize a full GP model with statistics on a sparse set of $M \ll N$ fictitious observations~\cite{quinonero2005unifying,titsias2009variational}. By representing a GP with these inducing points, the time and the space complexities are reduced to $O(NM^2+M^3)$ and $O(NM+M^2 )$, respectively. To further process datasets that are too large to fit into memory, stochastic approximations  have been proposed for regression~\cite{hensman2013gaussian} and classification~\cite{hensman2015scalable}. These methods have similar complexity bounds, but with $N$ replaced by the size of a mini-batch $N_m$.

Despite the success of sparse models, the scalability issues of GP inference are far from resolved. The major obstruction is that the cubic complexity in $M$ in the aforementioned upper-bound is also a lower-bound, which results from the inversion of an $M$-by-$M$ covariance matrix defined on the inducing points. As a consequence, these models can only afford to use a small set of $M$ basis functions, limiting the expressiveness of GPs for prediction.

In this work, we show that superlinear complexity is not completely necessary. Inspired by the reproducing kernel Hilbert space (RKHS) representation of GPs~\cite{cheng2016incremental}, we propose a generalized variational GP model, called \ourModel{}s (Decoupled Gaussian Processes), which \emph{decouples} the bases for the mean and the covariance functions. 
Specifically, let $M_\alpha$ and $M_\beta$ be the numbers of basis functions used to model the mean and the covariance functions, respectively. Assume $M_\alpha \geq M_\beta$. 
We show, when \ourModel{}s are used as a variational posterior~\cite{titsias2009variational}, the associated variational inference problem can be solved by stochastic gradient ascent with space complexity $O( N_m M_\alpha + M_\beta^2)$ and time complexity $O(D N_m M_{\alpha} + N_m M_\beta^2 + M_{\beta}^3)$, where $D$ is the input dimension. We name this algorithm~\ourMethod{}.  As a result, we can  choose $M_\alpha \gg M_\beta$, which allows us to keep the time and space complexity similar to previous methods (by choosing $M_\beta = M$) while greatly increasing accuracy.  To the best of our knowledge, this is the first variational GP algorithm that admits linear complexity in $M_\alpha $, without any assumption on the choice of kernel and likelihood.

While we design \ourMethod{} for general likelihoods, in this paper we study its effectiveness in Gaussian process regression (GPR) tasks.
We consider this is without loss of generality, as most of the sparse variational GPR algorithms in the literature can be modified to handle general likelihoods by introducing additional approximations
(e.g. in~\citet{hensman2015scalable} and ~\citet{sheth2015sparse}).  Our experimental results show that \ourMethod{} significantly outperforms the existing techniques, achieving  higher variational lower bounds and lower prediction errors when evaluated on held-out test sets.

\subsection{Related Work}
\label{sec:related work}

\begin{table*}[t]
{\footnotesize
\centering
\begin{tabular}{llllllll}
\toprule
&  $\mbf{a}$, $\mbf{B}$ & $\alpha$,$\beta$ & $\theta$ & $\alpha = \beta$ &  $N \neq M$ & Time  & Space\\
\midrule
\ourMethod{} & \textsc{sga} & \textsc{sga} & \textsc{sga} & \textsc{false} 
& \textsc{true}  & $O(D N M_{\alpha} + N M_\beta^2 + M_{\beta}^3)$ 
&$O( N M_\alpha + M_\beta^2)$  \\
\textsc{svi} & \textsc{snga} & \textsc{sga} & \textsc{sga} & \textsc{true}& \textsc{true} & $ O(D N M +  N  M^2 + M^3)$ & $O(NM + M^2)$
\\
i\textsc{vsgpr} & \textsc{sma} & \textsc{sma} & \textsc{sga} & \textsc{true}& \textsc{true} & $ O(D N M +  N M^2 + M^3)$ & $O(NM + M^2)$ \\
\textsc{vsgpr} & \textsc{cg} & \textsc{cg} & \textsc{cg} & \textsc{true}& \textsc{true} & $ O(DNM +  N M^2 + M^3)$ & $O(NM + M^2)$ \\
\textsc{gpr} & \textsc{cg} & \textsc{cg} & \textsc{cg} & \textsc{true}& \textsc{false} & $ O(DN^2 + N^3)$ & $O(N^2)$ \\
\bottomrule
\end{tabular}

\caption[Comparison]{Comparison between \ourMethod{} and  variational GPR algorithms:  \textsc{svi}~\cite{hensman2013gaussian},  i\textsc{vsgpr}~\cite{cheng2016incremental}, \textsc{vsgpr}~\cite{titsias2009variational}, and \textsc{gpr}~\cite{rasmussen2006gaussian}, where  $N$ is the number of observations/the size of a mini-batch, $M$, $M_\alpha$, $M_\beta$ are the number of basis functions, and $D$ is the input dimension. Here it is assumed $M_\alpha \geq M_\beta$ \protect\footnotemark 	
	.
} \label{tb:algorithms}
} 
\end{table*} 
\footnotetext{The first three columns show the algorithms to update the parameters:  \textsc{sga}/\textsc{snga}/\textsc{sma} denotes stochastic gradient/natural gradient/mirror ascent, and \textsc{cg} denotes batch nonlinear conjugate gradient ascent. The 4th and the 5th columns indicate whether the bases for mean and covariance are strictly shared, and whether a variational posterior can be used. The last two columns list the time and space complexity.
}

Our framework is based on the variational inference problem proposed by~\citet{titsias2009variational}, which treats the inducing points as variational parameters to allow  direct approximation of the true posterior. This is in contrast to \citet{seeger2003fast, snelson2005sparse, quinonero2005unifying}, and~\citet{quia2010sparse}, which all use inducing points as hyper-parameters of a degenerate prior. While both approaches have the same time and space complexity, the latter additionally introduces a large set of unregularized hyper-parameters and, therefore, is more likely to suffer from over-fitting~\cite{bauer2016understanding}.

In Table~\ref{tb:algorithms}, we compare \ourMethod{} with recent GPR algorithms in terms of the assumptions made and the time and space complexity. Each algorithm can be viewed as a special way to solve the maximization of the variational lower bound~\eqref{eq:Var GP obj in RKHS}, presented in Section~\ref{sec:Variational Inference of Gaussian Measures}. Our  algorithm \ourMethod{} generalizes the previous approaches to allow the basis functions for the mean and the covariance to be decoupled, so an approximate solution can be found by  stochastic gradient ascent in linear complexity.

\begin{figure}
\vspace{-8mm}
	\centering
	\subfloat[$M = 10$]{\includegraphics[width=0.31\linewidth]{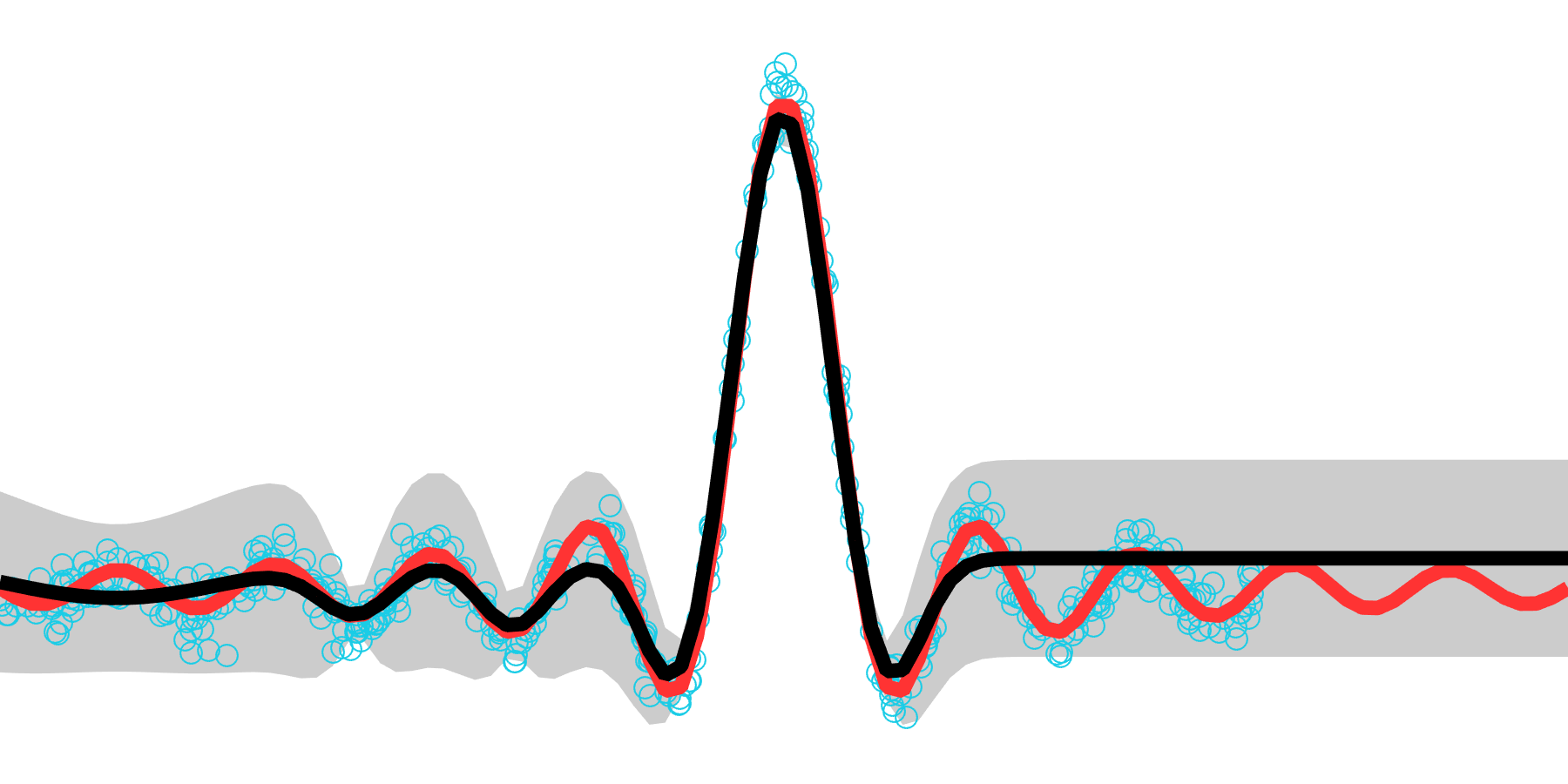}}
	\hspace{2mm}
	\subfloat[$M_\alpha = 100, M_\beta = 10$]{\includegraphics[width=0.31\linewidth]{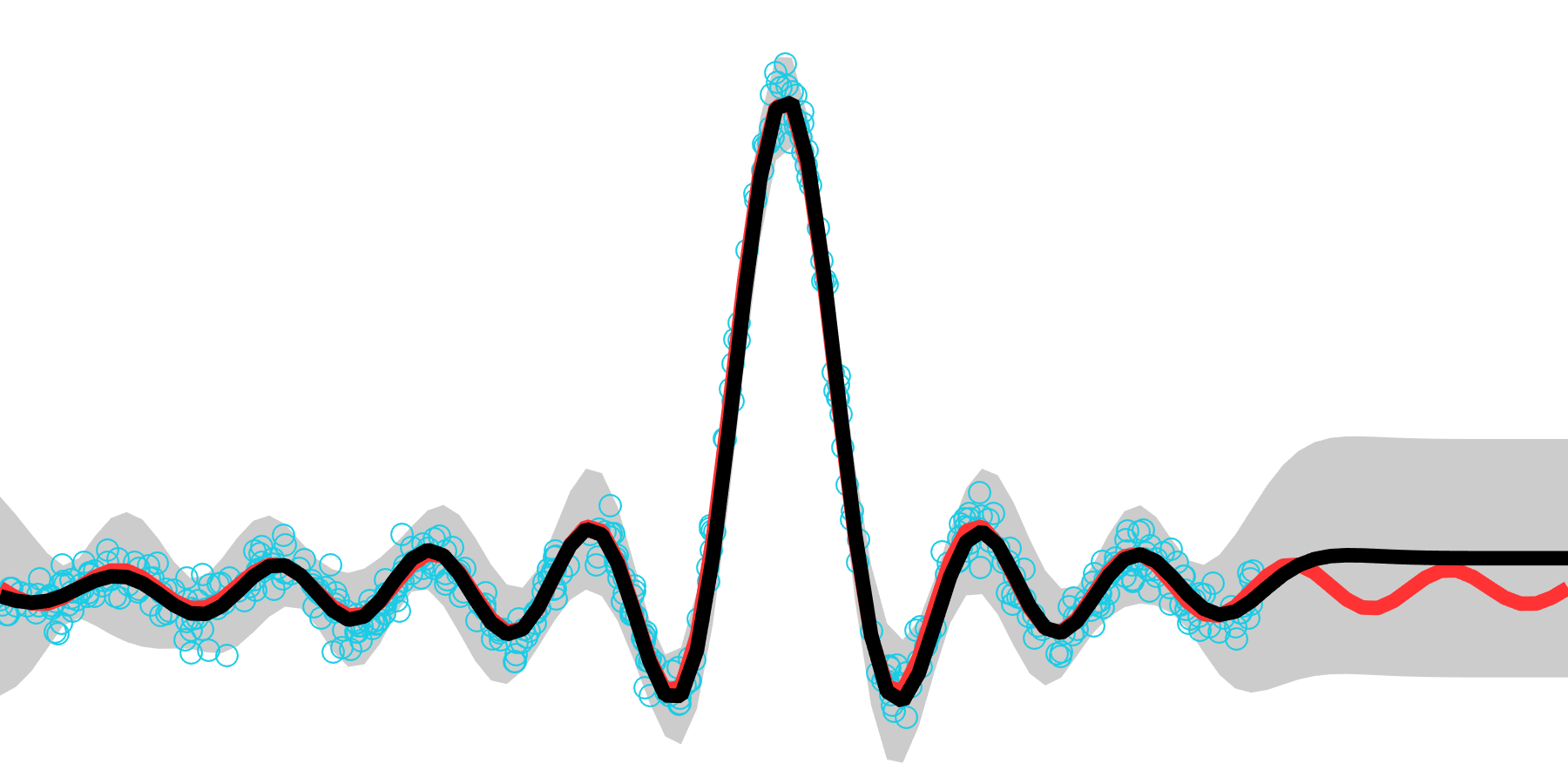}}
	\hspace{2mm}
	\subfloat[$M = 100$]{\includegraphics[width=0.31\linewidth]{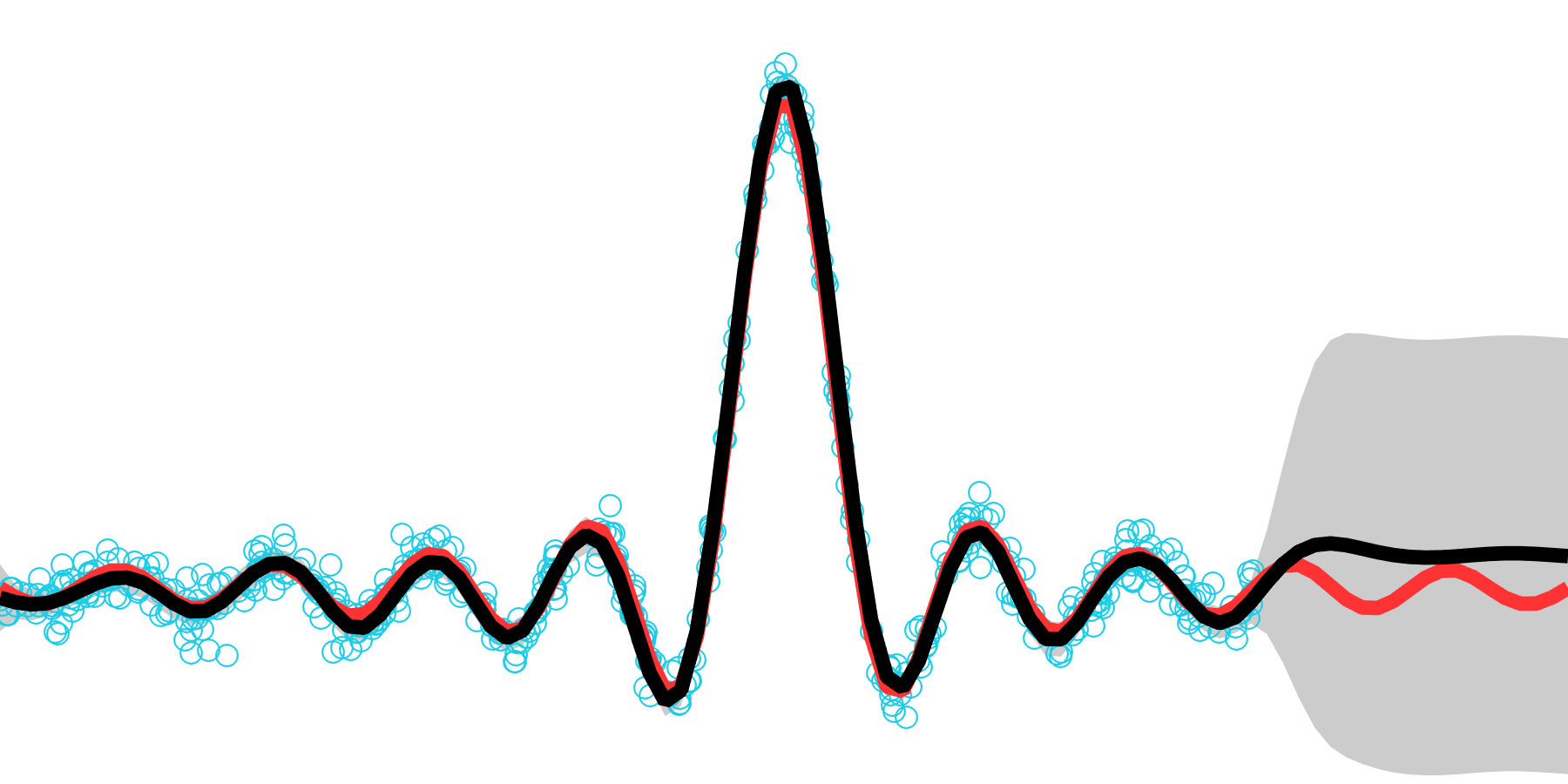}}
	\caption{Comparison between models with shared and decoupled basis. (a)(c) denote the models with shared basis of size $M$. (b) denotes the model of decoupled basis with size $(M_\alpha,M_\beta) $. In each figure, the red line denotes the ground truth; the blue circles denote the observations; the black line and the gray area denote the mean and variance in prediction, respectively. }
	\label{fig:toy example}
\end{figure}

To illustrate the idea, we consider a toy GPR example in Figure~\ref{fig:toy example}. The dataset contains 500 noisy observations of a $sinc$ function. Given the same training data, we conduct experiments with three different GP models. Figure~\ref{fig:toy example} (a)(c) show the results of the traditional coupled basis, which can be solved by any of the variational algorithms listed in Table~\ref{tb:algorithms}, and Figure~\ref{fig:toy example} (b) shows the result using the decoupled approach \ourMethod{}. The sizes of basis and observations are selected to emulate a large dataset scenario.  We can observe \ourMethod{} achieves a nice trade-off between prediction performance and complexity: it achieves almost the same accuracy in prediction as the full-scale model in Figure~\ref{fig:toy example}(c) and preserves the overall shape of the predictive variance. 

In addition to the sparse algorithms above, some recent attempts aim to revive the non-parametric property of GPs by structured covariance functions. For example, \citet{wilson2015kernel} proposes to space the inducing points on a multidimensional lattice, so the time and space complexities of using a product kernel becomes $O(N+DM^{1+1/D})$ and $O(N+DM^{1+2/D})$, respectively. However, because $M = c^D$, where $c$ is the number of grid points per dimension,  the overall complexity is exponential in $D$ and infeasible for high-dimensional data.
Another interesting approach by~\citet{hensman2016variational} combines variational inference~\cite{titsias2009variational} and a sparse spectral approximation~\cite{quia2010sparse}. By equally spacing inducing points on the spectrum, they show the covariance matrix on the inducing points have diagonal plus low-rank structure. With MCMC, the algorithm can achieve complexity $O(DNM)$. 
However, the proposed structure in~\cite{hensman2016variational} does not help to reduce the complexity when an approximate Gaussian posterior is favored or when the kernel hyper-parameters need to be updated.

Other kernel methods with linear complexity have been proposed using functional gradient descent ~\cite{kivinen2004online,dai2014scalable}. However, because these methods  use a model strictly the same size as the entire dataset, they fail to estimate the predictive covariance, which requires $\Omega(N^2)$ space complexity.  
Moreover, they cannot learn hyper-parameters online. 
The latter drawback also applies to greedy algorithms based on rank-one updates, e.g. the algorithm of \citet{csato2001sparse}. 

In contrast to these previous methods, our algorithm applies to \emph{all} choices of inducing points, likelihoods, and kernels, and we allow both  variational parameters and hyper-parameters to adapt online as more data are encountered.

\section{Preliminaries}
In this section, we briefly review the inference for GPs and the variational framework proposed by~\citet{titsias2009variational}. For now, we will focus on GPR for simplicity of exposition. We will discuss the case of general likelihoods in the next section when we introduce our framework,~\ourModel{}s.

\subsection{Inference for GPs} \label{sec:GP inference}

Let $f:\XX \to \R$ be a latent function defined on a compact domain $\XX \subset \R^D$. Here we assume \emph{a priori} that $f$ is distributed according to a Gaussian process $\GP(m,k)$. That is, $\forall x,x'\in \XX$, $\E[f(x)] = m(x)$ and $\Cbb[ f(x), f(x')] = k(x,x')$. In short, we write $f \sim \GP(m,k)$.

A GP probabilistic model is composed of a likelihood  $p(y | f(x))$ and a GP prior $\GP(m,k)$; in GPR, the likelihood is assumed to be Gaussian i.e. $ p(y | f(x)) = \NN( y | f(x), \sigma^2)$ with variance $\sigma^2$. Usually, the likelihood  and the GP prior are parameterized by some hyper-parameters, which we summarize as $\theta$. This includes, for example, the variance $\sigma^2$  and the parameters implicitly involved in defining $k(x,x')$. 
For notational convenience, and without loss of generality, we assume $m(x) = 0$ in the prior distribution and omit explicitly writing the dependence of distributions on $\theta$. 

Assume we are given a dataset $\DD = \{(x_n, y_n)\}_{n=1}^N$, in which $x_n\in \XX$ and $y_n \sim p(y | f(x_n))$. Let\footnote{In notation, we use boldface to distinguish finite-dimensional vectors (lower-case) and matrices (upper-case) that are used in computation from scalar and abstract mathematical objects.} $X = \{x_n\}_{n=1}^N$ and $ \mbf{y} = (y_n)_{n=1}^N $.  Inference for GPs involves solving for the posterior $p_{\theta^*}(f(x)| \mbf{y})$ for any  new  input $x\in \XX$, where $ \theta^* = \arg \max_{\theta} \log p_{\theta} (\mbf{y}) $. 
For example in GPR, because the likelihood is Gaussian, the predictive posterior  is also Gaussian with mean and covariance
 \begin{align}
 m_{|\mbf{y}}(x) = \mbf{k}_{x, X} ( \mbf{K}_X + \sigma^2 \mbf{I})^{-1} \mbf{y}, \qquad 
 k_{|\mbf{y}}(x,x') = k_{x, x'}  -  \mbf{k}_{x, X} ( \mbf{K}_X + \sigma^2 \mbf{I})^{-1}  \mbf{k}_{X,x'}, \label{eq:GPR post}
 \end{align}
and the hyper-parameter $\theta^*$ can be found by  nonlinear conjugate gradient ascent~\cite{rasmussen2006gaussian}
 \begin{align}
 \max_{\theta} \log p_\theta (\mbf{y}) = \max_{\theta} \log\NN( \mbf{y} | 0, \mbf{K}_X + \sigma^2 \mbf{I} ),
 \label{eq:full GP obj}
 \end{align}
where $k_{\cdot,\cdot}$, $\mbf{k_{\cdot,\cdot}}$ and $\mbf{K}_{\cdot,\cdot}$ denote the covariances between the sets in the subscript.\footnote{If the two sets are the same, only one is listed.}
One can show that these two functions, $m_{|\mbf{y}}(x)$ and $k_{|\mbf{y}}(x,x')$, define a valid GP. Therefore, given observations $\mbf{y}$, we say  $f \sim \GP(m_{|\mbf{y}},k_{|\mbf{y}})$. 

Although theoretically GPs are non-parametric and can model any function as $N \rightarrow \infty$, in practice this is difficult. 
As the inference has time complexity $\Omega(N^3)$ and space complexity $\Omega(N^2)$, applying vanilla GPs to large datasets is infeasible.

\subsection{Variational Inference with Sparse GPs}
To scale GPs to large datasets, \citet{titsias2009variational} introduced a scheme to compactly approximate the true posterior with a sparse GP,
$\GP( \hat{m}_{|\mbf{y}}, \hat{k}_{|\mbf{y}})$,
defined by the statistics on $M\ll N$ function values:
$
\{ L_m f(\tilde{x}_m) \}_{m=1}^M, 
$
where $L_m$ is a bounded linear operator\footnote{Here we use the notation $L_m f$ loosely for the compactness of writing. Rigorously, $L_m$ is a bounded linear operator acting on $m$ and $k$, not necessarily on all sample paths $f$.} and $\tilde{x}_m \in \XX$. 
$L_m f (\cdot)$ is called an \emph{inducing function} and $\tilde{x}_m$ an \emph{inducing point}. Common choices of $L_m$ include the identity map (as used originally by~\citet{titsias2009variational}) and integrals to achieve better approximation or to consider multi-domain information~\cite{walder2008sparse, figueiras2009inter,cheng2016learn}. Intuitively, we can think of $\{ L_m f(\tilde{x}_m) \}_{m=1}^M$ as a set of potentially indirect observations that capture salient information about the unknown function $f$.

\citet{titsias2009variational} solves for $\GP( \hat{m}_{|\mbf{y}}, \hat{k}_{|\mbf{y}})$ by variational inference. 
Let $\tilde{X} = \{\tilde{x}_m\}_{m=1}^M$ and let $\mbf{f}_X \in \R^N$ and $\mbf{f}_{\tilde{X}} \in \R^M$ be the (inducing) function values defined on $X$ and  $\tilde{X}$, respectively.
Let $p(\mbf{f}_{\tilde{X}})$ be the prior given by $\GP(m,k)$ and define $q( \mbf{f}_{\tilde{X}}) = \NN (\mbf{f}_{\tilde{X}} | \tilde{\mbf{m}},  \tilde{\mbf{S}})$ to be its variational posterior, where $\tilde{\mbf{m}} \in \R^M $ and $\tilde{\mbf{S}} \in \R^{M \times M}$ are the mean and the covariance of the approximate posterior of $\mbf{f}_{\tilde{X}}$. 
 \citet{titsias2009variational} proposes to use $q(\mbf{f}_X, \mbf{f}_{\tilde{X}}) =  p(\mbf{f}_X| \mbf{f}_{\tilde{X}} ) q( \mbf{f}_{\tilde{X}}) $ as the variational posterior to approximate $p(\mbf{f}_X, \mbf{f}_{\tilde{X}} | \mbf{y}) $ and to solve for $q( \mbf{f}_{\tilde{X}})$ together with the hyper-parameter $\theta$ through
\begin{align}
\max_{\theta, \tilde{X}, \tilde{\mbf{m}}, \tilde{\mbf{S}}  }   \LL_{\theta}( \tilde{X},\tilde{\mbf{m}}, \tilde{\mbf{S}}    ) =  \max_{\theta, \tilde{X}, \tilde{\mbf{m}}, \tilde{\mbf{S}}  }  \int q( \mbf{f}_X, \mbf{f}_{\tilde{X}}) \log \frac{ p( \mbf{y} | \mbf{f}_X) p(\mbf{f}_X |\mbf{f}_{\tilde{X}}) p(\mbf{f}_{\tilde{X}})  }{ q(\mbf{f}_X, \mbf{f}_{\tilde{X}})}  \der \mbf{f}_X \der \mbf{f}_{\tilde{X}}, \label{eq:Var GP obj}
\end{align}
where $ \LL_{\theta}$ is a variational lower bound of $ \log p_\theta(\mbf{y}) $, $p(\mbf{f}_X| \mbf{f}_{\tilde{X}} )  = \NN ( \mbf{f}_X|  \mbf{K}_{X, \tilde{X}} \mbf{K}_{\tilde{X}}^{-1} \mbf{f}_{\tilde{X}},  \mbf{K}_{X} -  \hat{\mbf{K}}_X  )$ is the conditional probability given in $\GP(m,k)$, and $\hat{\mbf{K}}_X = \mbf{K}_{X,\tilde{X}} \mbf{K}_{\tilde{X}}^{-1} \mbf{K}_{\tilde{X},X}$.

At first glance, the specific choice of variational posterior $q(\mbf{f}_X, \mbf{f}_{\tilde{X}})$ seems heuristic. However, although parameterized finitely, it resembles a full-fledged GP $\GP( \hat{m}_{|\mbf{y}}, \hat{k}_{|\mbf{y}})$:
\begin{align}
\hat{m}_{|\mbf{y}}(x) =  \mbf{k}_{x, \tilde{X}} \mbf{K}_{\tilde{X}}^{-1}\tilde{\mbf{m}},   \qquad
\hat{k}_{|\mbf{y}}(x,x') =  k_{x, x'}   +  \mbf{k}_{x, \tilde{X}} \mbf{K}_{\tilde{X}}^{-1} \left( \tilde{\mbf{S}} -  \mbf{K}_{\tilde{X}}  \right) \mbf{K}_{\tilde{X}}^{-1} \mbf{k}_{\tilde{X},x'}.  \label{eq:vsgpr}
\end{align}
This result is further studied in \citet{matthews2016sparse} and \citet{cheng2016incremental}, where it is shown that \eqref{eq:Var GP obj} is indeed minimizing a proper KL-divergence between Gaussian processes/measures.

By comparing~\eqref{eq:full GP obj} and~\eqref{eq:Var GP obj}, one can show that the time and the space complexities now reduce to $O(DNM + M^2N + M^3)$ and $O(M^2+MN)$, respectively, due to the low-rank structure of $\hat{\mbf{K}}_{\tilde{X}}$~\cite{titsias2009variational}. 
To further reduce complexity, stochastic optimization, such as stochastic natural ascent ~\cite{hensman2013gaussian} or stochastic mirror descent~\cite{cheng2016incremental} can be applied. In this case, $N$ in the above asymptotic bounds would be replaced by the size of a mini-batch $N_m$.  The above results can be modified to consider general likelihoods as in \citep{sheth2015sparse,hensman2015scalable}.

\section{Variational Inference with Decoupled Gaussian Processes}
Despite the success of sparse GPs, the scalability issues of GPs persist. Although parameterizing a GP with inducing points/functions enables learning from large datasets, it also restricts the expressiveness of the model. As the time and the space complexities still scale in $\Omega(M^3)$ and $\Omega(M^2)$, we cannot learn or use a complex model with large $M$.

In this work, we show that these two complexity bounds, which have long accompanied GP models, are not strictly necessary, but are due to the tangled representation canonically used in the GP literature. 
To elucidate this, we adopt the dual representation of~\citet{cheng2016incremental}, which treats GPs as linear operators in RKHS. But, unlike~\citet{cheng2016incremental},  we show how to decouple the basis representation of  mean and  covariance functions of a GP and  derive  a new variational problem, which can be viewed as a generalization of~\eqref{eq:Var GP obj}. We show that this problem---with arbitrary likelihoods and kernels---can be solved by stochastic gradient ascent with linear complexity in $M_\alpha$, the number of parameters used to specify the mean function for prediction. 

In the following, we first review the results in~\cite{cheng2016incremental}. We next introduce the decoupled representation, \ourModel{}s, and its variational inference problem. Finally, we present \ourMethod{} and discuss the case with general likelihoods.

\subsection{Gaussian Processes as Gaussian Measures}
\label{sec:Gaussian Measure}

Let an RKHS $\HH$ be a Hilbert space of functions with the reproducing property: $\forall x\in \XX$, $\exists \phi_x \in \HH$ such that $\forall f \in \HH$, $f(x) = \phi_x^T f$.\footnote{To simplify the notation, we write $\phi_x^T f$ for $ \langle f, \phi_x \rangle_\HH$, and $f^T L g$ for $ \langle f, L g \rangle_\HH$, where $f, g \in \HH$ and $L: \HH \to \HH$,  even if $\HH$ is infinite-dimensional.}
A Gaussian process $\GP(m,k)$ is equivalent to a Gaussian measure $\nu$ on Banach space $\BB$ which possesses an RKHS
$\HH$~\cite{cheng2016incremental}:\footnote{Such $\HH$ w.l.o.g. can be identified as the natural RKHS of the covariance function of a zero-mean prior GP.} there is a mean functional $ \mu \in \HH$ and a bounded positive semi-definite linear operator $\Sigma: \HH \to \HH$, such that for any $x,x' \in \XX$, $\exists \phi_x, \phi_{x'} \in \HH$,  we can write $m(x) =  \phi_x^T \mu$ and $k(x,x') =   \phi_x^T \Sigma \phi_{x'}$.
The triple $(\BB, \nu, \HH)$ is known as an abstract Wiener space~\cite{gross1967abstract,eldredge2016analysis}, in which $\HH$ is also called the Cameron-Martin space. 
Here the restriction that $\mu$, $\Sigma$ are RKHS objects is necessary, so the variational inference problem in the next section can be well-defined. 

We call this the \emph{dual} representation of a GP in RKHS $\HH$ (the mean function $m$ and the covariance function $k$ are realized as linear operators $\mu$ and $\Sigma$ defined in $\HH$).
With abuse of notation, we write $\NN (f|\mu, \Sigma)$ in short. {This notation does not mean a GP has a Gaussian distribution in $\HH$, nor does it imply that the sample paths from $\GP(m,k)$ are necessarily in $\HH$. Precisely, $\BB$ contains the sample paths of $\GP(m,k)$ and $\HH$ is dense in $\BB$.} In most applications of GP models, $\BB$ is the Banach space of continuous function $C(\XX;\YY)$ and $\HH$ is the span of the covariance function. As a special case, if $\HH$ is finite-dimensional, $\BB$ and $\HH$ coincide and $\nu$ becomes equivalent to a Gaussian distribution in a Euclidean space. 
 
In relation to our previous notation in Section~\ref{sec:GP inference}: suppose $k(x,x') = \phi_x^T \phi_{x'} $ and  $\phi_x: \XX \to \HH$ is a feature map to some Hilbert space $\HH$. Then we have assumed \emph{a priori} that $\GP(m,k) = \NN (f|0, I)$ is a normal Gaussian measure; that is $\GP(m,k)$ samples functions $f$ in the form $f(x) = \sum_{l=1}^{\dim \HH} \phi_l(x)^T \epsilon_l$, where $\epsilon_l \sim \NN(0,1)$ are independent. Note if $\dim \HH = \infty$, with probability one $f$ is not in $\HH$, but fortunately $\HH$ is large enough for us to approximate the sampled functions.
In particular, it can be shown that the posterior  $\GP(m_{|\mbf
y}, k_{|\mbf{y}})$ in GPR has a dual RKHS representation in the same RKHS as the prior GP~\cite{cheng2016incremental}.

\subsection{Variational Inference in Gaussian Measures} 
\label{sec:Variational Inference of Gaussian Measures}

\citet{cheng2016incremental} proposes a dual formulation of~\eqref{eq:Var GP obj} in terms of Gaussian measures\footnote{
We assume $q(f)$ is absolutely continuous wrt $p(f)$, which is true as $p(f)$ is non-degenerate. The integral denotes the expectation of {\scriptsize$\log p_\theta(y|f) + \log \frac{p(f)} {q(f)}$ } over $q(f)$,  and {\scriptsize$ \frac{q(f)} {p(f)}$} denotes the  Radon-Nikodym derivative.}:
\begin{align}
\max_{q(f),\theta}  \LL_\theta(q(f)) = 	\max_{q(f),\theta}  \int  q(f) \log \frac{p_\theta(y|f)p(f)} {q(f)}  \der{f}   =  \max_{q(f),\theta}\E_{q}[ \log  p_\theta(y|f) ] - \KL{q}{p},
\label{eq:Var GP obj in RKHS} 
\end{align}
where $q(f) = \NN(f|\tilde{\mu}, \tilde{\Sigma})$ is a variational Gaussian measure and $p(f) = \NN(f|0,I)$ is a normal prior. 
Its connection to the inducing points/functions in \eqref{eq:Var GP obj} can be summarized as follows~\cite{cheng2016incremental,cheng2016learn}:  Define a linear operator
$
\Psi_{\tilde{X}}: \R^M \to \HH 
$ as $ \mbf{a} \mapsto \sum_{m=1}^{M} a_m \psi_{\tilde{x}_m}$, where  $\psi_{\tilde{x}_m} \in \HH $ is defined such that $ \psi_{\tilde{x}_m}^T \mu = \E[ L_m f(\tilde{x}_m)] $. Then~\eqref{eq:Var GP obj} and~\eqref{eq:Var GP obj in RKHS} are equivalent, if $q(f)$ has a \emph{subspace parametrization},  
\begin{align} 
\tilde{\mu} &= \Psi_{\tilde{X}} \mbf{a}, \quad \tilde{\Sigma} =  I + \Psi_{\tilde{X}} \mbf{A} \Psi_{\tilde{X}}^T,\label{eq:subspace param}
\end{align}
with $\mbf{a}\in \R^M$ and $\mbf{A}\in \R^{M\times M}$ 
satisfying
$
\tilde{\mbf{m}} = \mbf{K}_{\tilde{X}} \mbf{a},
$ and $ 
\tilde{\mbf{S}} = \mbf{K}_{\tilde{X}} +  \mbf{K}_{\tilde{X}} \mbf{A} \mbf{K}_{\tilde{X}} 
$. In other words, the variational inference algorithms in the literature are all using a variational Gaussian measure in which $\tilde{\mu}$ and $\tilde{\Sigma}$ are parametrized by the same basis $\{ \psi_{\tilde{x}_m} | \tilde{x}_m \in \tilde{X} \}_{i=1}^M$.

Compared with~\eqref{eq:Var GP obj}, the formulation in~\eqref{eq:Var GP obj in RKHS} is neater: it follows the definition of the very basic variational inference problem. This is not surprising, since GPs can be viewed as Bayesian linear models in an infinite-dimensional space. Moreover, in~\eqref{eq:Var GP obj in RKHS} all hyper-parameters are isolated in the likelihood $p_{\theta}(y|f)$, because the prior is fixed as a normal Gaussian measure.

\subsection{Disentangling the GP Representation with \ourModel{}s} \label{sec:Disentanglement of Representation}

While \citet{cheng2016incremental} treat~\eqref{eq:Var GP obj in RKHS} as an equivalent form of~\eqref{eq:Var GP obj}, here we show that it is a generalization. By further inspecting~\eqref{eq:Var GP obj in RKHS}, it is apparent that sharing the basis $\Psi_{\tilde{X}}$ between $\tilde{\mu}$ and $\tilde{\Sigma}$ in~\eqref{eq:subspace param} is not strictly necessary, since~\eqref{eq:Var GP obj in RKHS} seeks to optimize two linear operators, $\tilde{\mu}$ and $\tilde{\Sigma}$. With this in mind, we propose a new parametrization that \emph{decouples} the bases for $\tilde{\mu}$ and $\tilde{\Sigma}$: 
\begin{align} 
\tilde{\mu} &= \Psi_{\alpha} \mbf{a}, \quad \tilde{\Sigma} =  (I + \Psi_{\beta} \mbf{B} \Psi_{\beta}^{T} )^{-1} \label{eq:decoupled subspace param}
\end{align}
where $\Psi_{\alpha}:\R^{M_\alpha} \to \HH$ and $\Psi_{\beta}: \R^{M_\beta} \to \HH$ denote linear operators defined similarly to $\Psi_{\tilde{X}}$ and $\mbf{B}  \succeq 0 \in \R^{M_\beta \times M_\beta}$. Compared with \eqref{eq:subspace param}, here we parametrize $\tilde{\Sigma}$ through its inversion with $\mbf{B}$ so the condition that $\tilde{\Sigma} \succeq 0$ can be easily realized as $\mbf{B} \succeq 0 $. This form agrees with the posterior covariance in GPR~\cite{cheng2016incremental} and will give a posterior that is strictly less uncertain than the prior. Note the choice of decoupled parametrization is not unique. In particular, the bases can be partially shared, or $(\mbf{a}, \mbf{B})$ can be further parametrized (e.g. $\mbf{B}$ can be parametrized using the canonical form in~\eqref{eq:vsgpr}) to improve the numerical convergence rate. Please refer to Appendix~\ref{app:practical notes} for a discussion.\footnote{Appendix~\ref{app:practical notes} is partially based on a discussion with Hugh Salimbeni at the NIPS conference. Here we adopt the fully decoupled, directly parametrized form in \eqref{eq:decoupled subspace param} to demonstrate the idea. We leave the full comparison of different decoupled parametrizations in future work.
}

The decoupled subspace  parametrization \eqref{eq:decoupled subspace param} corresponds to a \ourModel{}, $\GP(\hat{m}_{|\mbf{y}}^{\alpha}, \hat{k}_{|\mbf{y}}^{\beta})$, with mean and covariance functions as 
\footnote{In practice, we can parametrize  $\mbf{B} = \mbf{L}\mbf{L}^T $ with Cholesky factor $\mbf{L}  \in \R^{M_\beta \times M_\beta} $ so the problem is unconstrained. The required terms in \eqref{eq:decoupled GP} and later in \eqref{eq:KL} can be stably computed as $ \left( \mbf{B}^{-1} + \mbf{K}_{\beta}  \right)^{-1} = \mbf{L} \mbf{H}^{-1} \mbf{L}^T  $ and $\log |\mbf{I}  + \mbf{K}_{{\beta}} \mbf{B}| = \log |\mbf{H}| $, where  $\mbf{H} = \mbf{I} + \mbf{L}^T \mbf{K}_{\beta}\mbf{L} $.}
\begin{align}
\hat{m}_{|\mbf{y}}^{\alpha}(x) =  \mbf{k}_{x, \alpha} \bm{a}, \qquad
\hat{k}_{|\mbf{y}}^{\beta}(x,x') =  k_{x, x'}   -  \mbf{k}_{x,\beta}  \left( \mbf{B}^{-1} + \mbf{K}_{\beta}  \right)^{-1}  \mbf{k}_{\beta,x'}. 
\label{eq:decoupled GP}
\end{align}

While the structure of~\eqref{eq:decoupled GP} looks similar to~\eqref{eq:vsgpr}, directly replacing the basis $\tilde{X}$ in~\eqref{eq:vsgpr} with $\alpha$ and $\beta$ is not trivial. Because  the equations in~\eqref{eq:vsgpr} are derived from the traditional viewpoint of GPs as statistics on function values, the original optimization problem~\eqref{eq:Var GP obj} is not defined if $\alpha \neq \beta$ and therefore, it is not clear how to learn a decoupled representation traditionally. Conversely, by using the dual RKHS representation, the objective function to learn~\eqref{eq:decoupled GP} follows naturally from~\eqref{eq:Var GP obj in RKHS}, as we will show next. 

\subsection{\ourMethod{}: Algorithm and Analysis} 

Substituting the decoupled subspace parametrization~\eqref{eq:decoupled subspace param} into the variational inference problem in ~\eqref{eq:Var GP obj in RKHS} results in a numerical optimization problem: $\max_{q(f),\theta}\E_{q}[ \log  p_\theta(y|f) ] - \KL{q}{p}$  with 
\begin{align} 
\KL{q}{p} &= \frac{1}{2} \mbf{a}^T \mbf{K}_{{\alpha}}  \mbf{a} 
+ \frac{1}{2} \log |\mbf{I}  + \mbf{K}_{{\beta}} \mbf{B}|  + \frac{-1}{2}  \tr{ \mbf{K}_{\beta}(  \mbf{B}^{-1}  + \mbf{K}_{{\beta}})^{-1}}       \label{eq:KL}    \\
\E_{q}[ \log  p_\theta(y|f) ] &= \sum_{n=1}^{N} \E_{q({f(x_n)})}[ \log  p_\theta(y_n|f(x_n)) ] \label{eq:expected log-likeli} 
\end{align}
where each expectation is over a scalar Gaussian $q({f(x_n)})$ given by~\eqref{eq:decoupled GP} as functions of $(\mbf{a}, \alpha)$ and $(\mbf{B}, \beta)$. Our objective function contains \cite{hensman2015scalable} as a special case, which assumes $\alpha = \beta = \tilde{X}$. In addition, we note that ~\citet{hensman2015scalable} indirectly parametrize the posterior by $\tilde{\mbf{m}}$ and $\tilde{\mbf{S}}= \mbf{L}\mbf{L}^T$, whereas we parametrize directly by \eqref{eq:subspace param} with $\mbf{a}$ for scalability and $\mbf{B} = \mbf{L}\mbf{L}^T $ for better stability (which always reduces the uncertainty in the posterior compared with the prior).

We notice that $(\mbf{a},\alpha)$ and $(\mbf{B},\beta)$ are completely decoupled in~\eqref{eq:KL} and \emph{potentially} combined again in \eqref{eq:expected log-likeli}. In particular, if $  p_\theta(y_n|f(x_n))$ is Gaussian as in GPR, we have an additional decoupling, i.e.
$
\LL_{\theta}(\mbf{a},\mbf{B},\alpha, \beta) = \FF_{\theta}(\mbf{a},\alpha)  + \GG_{\theta}(\mbf{B}, \beta)
$
for some $\FF_{\theta}(\mbf{a},\alpha)$  and $\GG_{\theta}(\mbf{B}, \beta)$. 
Intuitively, the optimization over $(\mbf{a}, \alpha)$ aims to minimize the fitting-error, and the optimization over $(\mbf{B},\beta)$ aims to memorize the samples encountered so far; the mean and the covariance functions only interact indirectly through the optimization of the hyper-parameter $\theta$. 

One salient feature of~\ourMethod{} is that it tends to overestimate, rather than underestimate, the variance, when we select $M_\beta \leq M_\alpha$. This is inherited from the non-degeneracy property of the variational framework~\cite{titsias2009variational} and can be seen in the toy example in Figure~\ref{fig:toy example}. In the extreme case when $M_\beta = 0$, we can see the covariance in~\eqref{eq:decoupled GP} becomes the same as the prior; moreover, the objective function of~\ourMethod{} becomes similar to kernel methods (exactly the same as kernel ridge regression, when the likelihood is Gaussian). The additional inclusion of expected log-likelihoods here allows \ourMethod{} to learn the hyper-parameters in a unified framework, as its objective function can be viewed as minimizing a generalization upper-bound in PAC-Bayes learning~\citep{germain2016pac}. 

\ourMethod{} solves the above optimization problem by stochastic gradient ascent. 
Here we purposefully ignore specific details of $p_\theta(y|f)$ to emphasize that~\ourMethod{} can be applied to general likelihoods as it only requires unbiased first-order information, which e.g. can be found in~\citep{sheth2015sparse}. 
In addition to having a more adaptive representation, the main benefit of~\ourMethod{} is that the computation of an unbiased gradient requires only linear complexity in $M_\alpha$, as shown below (see Appendix~\ref{app:VIDGP}for details).

\paragraph{KL-Divergence}
Assume  $|\alpha| = O(DM_\alpha)$ and $|\beta| = O(DM_\beta)$.  
By~\eqref{eq:KL}, One can show 
$ \nabla_{\mbf{a}} \KL{q}{p} =  \mbf{K}_{{\alpha}}  \mbf{a}$ and $\nabla_{\mbf{B}} \KL{q}{p} = \frac{1}{2} ( \mbf{I}  + \mbf{K}_{{\beta}} \mbf{B})^{-1}  \mbf{K}_{{\beta}} \mbf{B} \mbf{K}_{{\beta}}  (\mbf{I}  + \mbf{B}\mbf{K}_{{\beta}}  )^{-1}. $
Therefore, the time complexity to compute $\nabla_{\mbf{a}} \KL{q}{p}$ can be reduced to $O( N_m M_\alpha)$ if we sample over the columns of $\mbf{K}_{\alpha}$ with a mini-batch of size $N_m$. 
By contrast, the time complexity to compute $\nabla_{\mbf{B}} \KL{q}{p}$ will always be $\Theta(M_\beta^3)$ and cannot be further reduced, regardless of the parametrization of $\mbf{B}$.\footnote{Due to $ \mbf{K}_{{\beta}}$, the complexity would remain as $O(M_\beta^3)$ even if $\mbf{B}$ is constrained to be diagonal.}
The gradient with respect to $\alpha$ and $\beta$ can be derived similarly and have time complexity $O( D N_m M_\alpha)$ and  $O(D  M_\beta^2 + M_\beta^3)$, respectively.

\paragraph{Expected Log-Likelihood}
Let  $\hat{\mbf{m}}(\mbf{a}, \alpha) \in \R^N$ and $\hat{\mbf{s}}(\mbf{B}, \beta) \in \R^N$ be the vectors of the mean and covariance of scalar Gaussian $q({f(x_n)})$ for $n \in \{1, \dots, N\}$.
As~\eqref{eq:expected log-likeli} is a sum over $N$ terms, by sampling with a mini-batch of size $N_m$, an unbiased gradient of~\eqref{eq:expected log-likeli} with respect to $(\theta, \hat{\mbf{m}}, \hat{\mbf{s}})$ can be computed in $O(N_m)$. To compute the full gradient with respect to $(\mbf{a},\mbf{B},\alpha, \beta)$, we compute the derivative of $\hat{\mbf{m}}$ and $\hat{\mbf{s}}$ with respect to $(\mbf{a},\mbf{B},\alpha, \beta)$ and then apply chain rule. These steps take  $O( DN_m  M_\alpha)$ and $O(D N_m M_\beta+ N_m M_\beta^2 + M_\beta^3)$ for $(\mbf{a},\alpha)$ and $(\mbf{B},\beta)$, respectively.

The above analysis shows that the curse of dimensionality in GPs originates in the covariance function. For space complexity, the decoupled  parametrization~\eqref{eq:decoupled subspace param} requires  memory in $O( N_m M_\alpha + M_\beta^2)$; for time complexity, an unbiased  gradient with respect to  $(\mbf{a},\alpha)$ can be computed in $O(D N_m M_\alpha)$, but that with respect to $(\mbf{B},\beta)$  has time complexity $\Omega(D N_m M_\beta+ N_m M_\beta^2 + M_\beta^3)$.  This motivates choosing $M_\beta = O(M)$ and $M_\alpha$ in $O(M_\beta^{2} )$ or $O(M_\beta^{3})$, which maintains the same complexity as previous variational techniques but greatly improves the prediction performance.

\begin{algorithm}[t]
	{\footnotesize
		\caption{Online Learning with \ourModel{}s}\label{alg:algo}
		\begin{algorithmic} [1]
			\renewcommand{\algorithmicrequire}{\textbf{Parameters:}}
			\renewcommand{\algorithmicensure}{\textbf{Input:}}		
			\REQUIRE $M_\alpha$, $M_\beta$, $N_m$, $N_{\Delta}$
			\ENSURE  $\MM(\mbf{a},\mbf{B},\alpha,\beta, \theta)$ , $\DD$ 
			\STATE $\theta_0$  $\leftarrow$ initializeHyperparameters(  sampleMinibatch($\DD$, $N_m$) )
			\FOR{$t = 1\dots T$}
			\STATE $D_t$  $\leftarrow$ sampleMinibatch($\DD$, $N_m$)
			\STATE $\MM$.addBasis($D_t$, $N_{\Delta}$, $M_\alpha$, $M_\beta$)
			\STATE $\MM$.updateModel($D_t$, t)
			\ENDFOR
		\end{algorithmic}
	} 
\end{algorithm}

\section{Experimental Results} \label{sec:experiments}
We compare our new algorithm, \ourMethod{}, with the state-of-the-art incremental algorithms for sparse variational GPR, \textsc{svi}~\cite{hensman2013gaussian} and i\textsc{vsgpr}~\cite{cheng2016incremental}, as well as the classical \textsc{gpr} and the batch algorithm \textsc{vsgpr}~\citep{titsias2009variational}. As discussed in Section~\ref{sec:related work}, these methods can be viewed as different ways to optimize~\eqref{eq:Var GP obj in RKHS}. Therefore, in addition to the normalized mean square error (nMSE)~\cite{rasmussen2006gaussian} in prediction,
we report the performance in the variational lower bound (VLB)~\eqref{eq:Var GP obj in RKHS}, which also captures the quality of the predictive variance and hyper-parameter learning.\footnote{The exact marginal likelihood is computationally infeasible to evaluate for our large model.} These two metrics are evaluated on  held-out test sets in all of our experimental domains.

Algorithm~\ref{alg:algo} summarizes the online learning procedure used by all stochastic algorithms,\footnote{The algorithms differs only in whether the bases are shared and how the model is updated (see Table~\ref{tb:algorithms}).} where each learner has to optimize all the parameters on-the-fly using \emph{i.i.d.} data. The hyper-parameters are first initialized heuristically by median trick using the first mini-batch. We incrementally build up the variational posterior by including $N_{\Delta}\leq N_m$ observations in each mini-batch as the \emph{initialization} of new variational basis functions. Then all the hyper-parameters and the variational parameters are updated online. These steps are repeated for $T$ iterations.

For all the algorithms, we assume the prior covariance is defined by the \textsc{se-ard} kernel~\cite{rasmussen2006gaussian} and we use the generalized \textsc{se-ard} kernel~\cite{cheng2016incremental} as the inducing functions in the variational posterior (see Appendix~\ref{app:exp setup} for details). We note that all algorithms in comparison use the same kernel and optimize both the variational parameters (including inducing points) and the hyperparameters.

In particular,  we implement \textsc{sga} by \textsc{adam}~\cite{kingma2014adam} (with default parameters  $\beta_1 =0.9$ and $\beta_2 = 0.999$). The step-size for each stochastic algorithms is scheduled according to $\gamma_t = \gamma_0 ( 1+ 0.1 \sqrt{t})^{-1}$, where $\gamma_0 \in \{10^{-1}, 10^{-2}, 10^{-3}\}$ is selected manually for each algorithm to maximize the improvement in objective function after the first 100 iterations. 
We test each stochastic algorithm for $T =2000$ iterations with mini-batches of size $N_m = 1024$ and the increment size  $N_\Delta = 128$. Finally, the model sizes used in the experiments are listed as follows: $M_\alpha = 128^2$ and $M_\beta = 128$ for \ourMethod{}; $M = 1024$ for \textsc{svi}; $M = 256$ for i\textsc{vsgpr}; $M= 1024$, $N = 4096$ for \textsc{vsgpr}; $N = 1024$ for \textsc{gp}. These settings share similar order of time complexity in our current Matlab implementation.

\subsection{Datasets}

\paragraph{Inverse Dynamics of KUKA Robotic Arm} This dataset records the inverse dynamics of a KUKA arm performing rhythmic motions at various speeds~\cite{meier2014incremental}. The original dataset consists of two parts: \textsc{kuka}$_1$ and \textsc{kuka}$_2$, each of which have 17,560 offline data and 180,360 online data with 28 attributes and 7 outputs. In the experiment, we mix the online and the offline data and then split 90\% as training data (178,128 instances) and 10\% testing data (19,792 instances) to satisfy the \emph{i.i.d.} assumption.

\paragraph{Walking MuJoCo}
MuJoCo (Multi-Joint dynamics with Contact) is a physics engine for research in robotics, graphics, and animation, created by~\citep{todorov2012mujoco}. In this experiment, we gather 1,000 walking trajectories by running \textsc{trpo}~\cite{schulman2015trust}. In each time frame, the MuJoCo transition dynamics have a 23-dimensional input and a 17-dimensional output. 
We consider two regression problems to predict 9 of the 17 outputs from the input\footnote{Because of the structure of MuJoCo dynamics, the rest 8 outputs can be trivially known from the input.}: \textsc{mujoco}$_1$ which maps the input of the current frame (23 dimensions) to the output, and \textsc{mujoco}$_2$ which maps the inputs of the current and the previous frames (46 dimensions) to the output. In each problem, we randomly select 90\% of the data as training data (842,745 instances) and 10\% as test data (93,608 instances).

\begin{table*}[!t] \vspace{-5mm}
	\centering
	{\scriptsize
		\subfloat{
			\begin{tabular}{cccccc}
				\multicolumn{6}{c}{ \textsc{kuka}$_1$ -  Variational Lower Bound ($10^5$) } \\
				\toprule
				& \ourMethod{}  & \textsc{svi}   &  i\textsc{vsgpr} & \textsc{vsgpr} & \textsc{gpr} \\
				\midrule	  
				mean & \textbf{1.262} & 0.391 & 0.649 & 0.472 & -5.335\\   
				std & \textbf{0.195} & 0.076 & 0.201 & 0.265 & 7.777\\   
				\bottomrule
			\end{tabular}} 
			\hspace{5mm}
			\subfloat{
				\begin{tabular}{cccccc}
					\multicolumn{6}{c}{ \textsc{kuka}$_1$ - Prediction Error (nMSE) } \\
					\toprule
					& \ourMethod{}  & \textsc{svi}   &  i\textsc{vsgpr} & \textsc{vsgpr} & \textsc{gpr} \\    
					\midrule
					mean & \textbf{0.037} & 0.169 & 0.128 & 0.139 & 0.231\\    
					std & \textbf{0.013} & 0.025 & 0.033 & 0.026 & 0.045\\     
					\bottomrule
				\end{tabular}} 
				\\
				\subfloat{
					\begin{tabular}{cccccc}
						\multicolumn{6}{c}{  \textsc{mujoco}$_1$ -  Variational Lower Bound ($10^5$) } \\
						\toprule
						& \ourMethod{}  & \textsc{svi}   &  i\textsc{vsgpr} & \textsc{vsgpr} & \textsc{gpr} \\
						\midrule
						mean & \textbf{6.007} & 2.178 & 4.543 & 2.822 & -10312.727\\   
						std & \textbf{0.673} & 0.692 & 0.898 & 0.871 & 22679.778\\    
						\bottomrule
					\end{tabular}} 
					\hspace{5mm}
					\subfloat{
						\begin{tabular}{cccccc}
							\multicolumn{6}{c}{ \textsc{mujoco}$_1$ - Prediction Error (nMSE) } \\
							\toprule
							& \ourMethod{}  & \textsc{svi}   &  i\textsc{vsgpr} & \textsc{vsgpr} & \textsc{gpr} \\         
							\midrule 
							mean & \textbf{0.072} & 0.163 & 0.099 & 0.118 & 0.213\\    
							std & \textbf{0.013} & 0.053 & 0.026 & 0.016 & 0.061\\     
							\bottomrule
						\end{tabular}} 
					}
					\caption{Experimental results of \textsc{kuka}$_1$ and \textsc{mujoco}$_1$ after 2,000 iterations.} \label{tb:results}
				\end{table*} 

\begin{figure*}[t]	
	\centering
	\subfloat[Sample Complexity]{\includegraphics[width=0.4\textwidth]{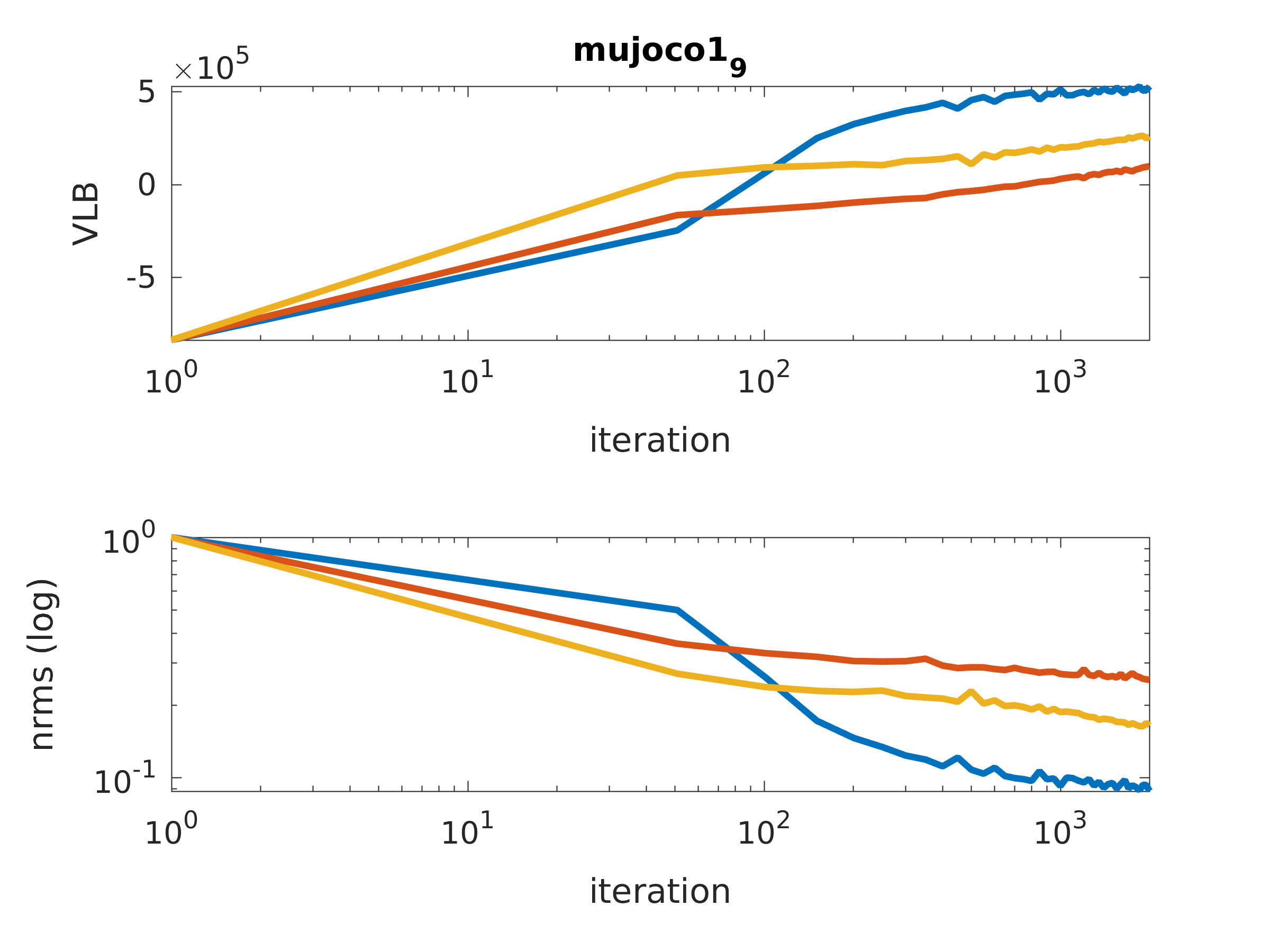} } \qquad \subfloat[Time Complexity]{\includegraphics[width = 0.4\textwidth]{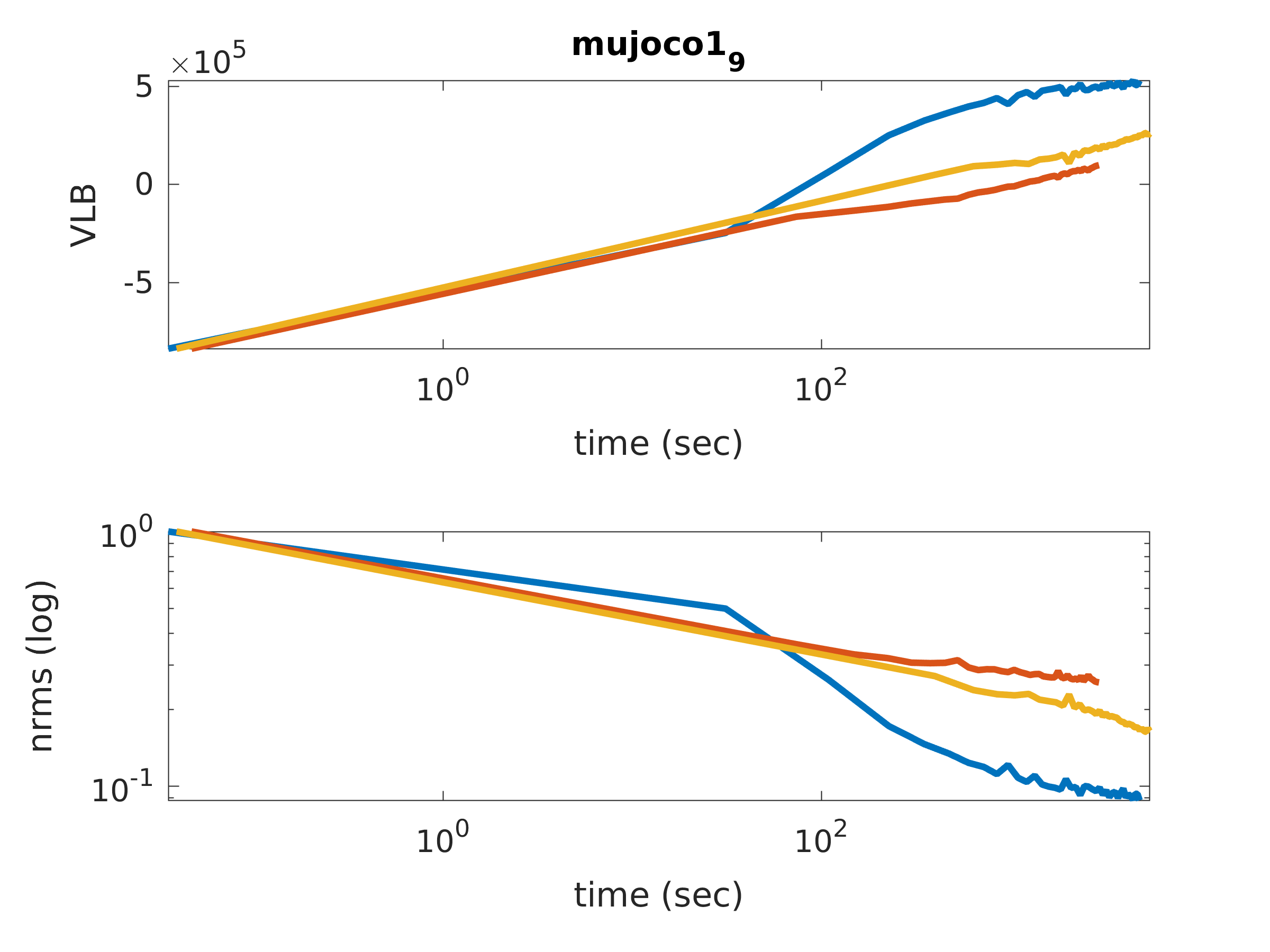}}
	\caption{An example of online learning results  (the 9th output of \textsc{mujoco}$_1$  dataset). The blue, red, and yellow lines denote \ourMethod, \textsc{svi}, and i\textsc{vsgpr}, respectively.} \label{fig:Online Learning Results}
\end{figure*}

\subsection{Results} 
We summarize part of the experimental results in Table~\ref{tb:results} in terms of nMSE in prediction and VLB.
While each output is treated independently during learning, Table~\ref{tb:results}  present the mean and the standard deviation over all the outputs as the selected metrics are normalized. For the complete experimental results, please refer to  Appendix~\ref{app:exp results}.

We observe that \ourMethod{} consistently outperforms the other approaches with much higher VLBs and much lower prediction errors; \ourMethod{} also has smaller standard deviation.
These results validate our initial hypothesis that adopting a large set of basis functions for the mean  can  help when modeling complicated functions. 
i\textsc{vsgpr} has the next best result after \ourMethod{}, despite using a basis size of 256, much smaller than that of 1,024 in \textsc{svi}, \textsc{vsgpr}, and \textsc{gpr}. Similar to \ourMethod{},  i\textsc{vsgpr} also generalizes better than the batch algorithms \textsc{vsgpr} and \textsc{gpr}, which only have access to a smaller set of training data and are more prone to over-fitting. 
By contrast, the performance of \textsc{svi} is surprisingly worse than \textsc{vsgpr}. We conjecture this might be due to the fact that the hyper-parameters and the inducing points/functions are only crudely initialized in online learning. We additionally find that the stability of \textsc{svi} is more sensitive to the choice of step size than other methods.
This might explain why in~\citep{hensman2013gaussian,cheng2016incremental} batch data was used to initialize the hyper-parameters and the learning rate to update the hyper-parameters was selected to be much smaller than that for stochastic natural gradient ascent.

To further investigate the properties of different stochastic approximations, we show the change of VLB and the prediction error over iterations and time in Figure~\ref{fig:Online Learning Results}.
Overall, whereas~i\textsc{vsgpr} and \textsc{svi} share similar convergence rate, the behavior of \ourMethod{} is different. 
We see that i\textsc{vsgpr} converges the fastest, both in time and sample complexity. Afterwards, \ourMethod{} starts to descend faster and surpass the other two  methods. From Figure~\ref{fig:Online Learning Results}, we can also observe that although~\textsc{svi} has similar convergence to i\textsc{vsgpr}, it slows down earlier and therefore achieves a worse result. 
These phenomenon are observed in multiple experiments.

\section{Conclusion}
We propose a novel, fully-differentiable framework, Decoupled Gaussian Processes \ourModel{}s, for large-scale GP problems. By decoupling the representation, we derive a variational inference problem that can be solved with stochastic gradients with \emph{linear} time and space complexity. Compared with existing  algorithms, \ourMethod{} can adopt a much larger set of basis functions to predict more accurately. Empirically, \ourMethod{} significantly outperforms state-of-the-arts variational sparse GPR algorithms in multiple regression tasks. These encouraging experimental results motivate further application of \ourMethod{} to end-to-end learning with neural networks in large-scale, complex real world problems.

\subsubsection*{Acknowledgments}
This work was supported in part by NSF NRI award 1637758. The authors additionally thank the reviewers and Hugh Salimbeni for productive discussion which improved the quality of the paper.

{
\bibliography{ref}
\bibliographystyle{plainnat}}

\ifCompileAppendix
\clearpage
\onecolumn
\appendix
\section*{Appendix}

\section{Subspace Parametrization and Notes to Practitioners} \label{app:practical notes}

As mentioned in Section~\ref{sec:Disentanglement of Representation}, the choice of subspace parametrization is non-unique and not limited to 
\begin{align*} 
\tilde{\mu} &= \Psi_{\alpha} \mbf{a}, \quad \tilde{\Sigma} =  (I + \Psi_{\beta} \mbf{B} \Psi_{\beta}^{T} )^{-1}. 
\end{align*}
While we adopted the completely decoupled, directly parametrized representation in the experiments for its simplicity and sufficiency to validate our idea, this choice may not possess the best numerical properties. Here we point out other potential, practical parameterizations. A complete study of these choices is outside of the scope of this paper.

\subsection{Numerical Convergence Issues}
To understand the effect of the parametrization on the convergence rate,  intuition can be gained by inspecting the objective function
$\max_{q(f),\theta}\E_{q}[ \log  p_\theta(y|f) ] - \KL{q}{p}$, where 
\begin{align*} 
\KL{q}{p} &= \frac{1}{2} \mbf{a}^T \mbf{K}_{{\alpha}}  \mbf{a} 
+ \frac{1}{2} \log |\mbf{I}  + \mbf{K}_{{\beta}} \mbf{B}|  + \frac{-1}{2}  \tr{ \mbf{K}_{\beta}(  \mbf{B}^{-1}  + \mbf{K}_{{\beta}})^{-1}}       \\
\E_{q}[ \log  p_\theta(y|f) ] &= \sum_{n=1}^{N} \E_{q({f(x_n)})}[ \log  p_\theta(y_n|f(x_n)) ] 
\end{align*}
When the likelihood function is Gaussian, this leads to a problem which is quadratic in $\mbf{a}$. Therefore, the numerical convergence rate can be slow, for example, when $\mbf{K}_{\alpha}$ is ill-conditioned (which can happen especially  when $M_{\alpha}$ is large and $\Psi_{\alpha}$ is not flexible and correlated). Empirically, we observed a slow-down of convergence rate when the standard $\textsc{se-ard}$ kernel was used as the variational basis, compared with the generalized $\textsc{se-ard}$ kernel used in the experiments.

The use of a preconditioner can help the convergence of first-order methods in practice.  This can be achieved by further parameterizing $(\mbf{a}, \mbf{B})$.  For example, a Jacobi preconditioner can be optionally used in implementation by parameterizing $\mbf{a}$ and $\mbf{L}$ (i.e. $\mbf{B}= \mbf{L} \mbf{L}^T$) through $\mbf{a}_0$ and $\mbf{L}_0$ as
\begin{align*}
\mbf{a} = \diag(\mbf{K}_{\alpha})^{-1} \mbf{a}_0,  \qquad 	\mbf{L} = \diag(\mbf{K}_{\beta})^{-1} \mbf{L}_0
\end{align*}
where $\diag$ denotes  the diagonal part. While the Jacobi preconditioner is only $\rho^{-2} \mbf{I}$ for \textsc{se-ard} kernels for some scaling constant $\rho \in \R$, we observed it still helps the convergence rate as the hyperparameter $\rho$ is also being updated online.

Comparing the Jacobi preconditioner and~\eqref{eq:vsgpr}, we can see that the canonical parametrization in~\eqref{eq:vsgpr}  can be viewed as a full preconditioner, which can enjoy a better numerical convergence rate but at the cost of higher computational complexity. 
For the mean function, this type of full preconditioner, with $\mbf{K}_{\alpha}^{-1}$, is too expensive to compute when $M_{\alpha}$ is large; nonetheless, using an approximation, such as the Nystr\"om approximation or (incomplete) Cholesky decomposition, can potentially improve convergence without increasing the computational complexity. For the covariance function, because $M_\beta$ is small, a full preconditioner can be used, such as setting $\mbf{L} = \mbf{K}_{\beta}^{-1} \mbf{L}_0$ or parameterizing $\mbf{B}$ with the canonical choice~\eqref{eq:vsgpr} (see Appendix~\ref{app:hybrid subspace param} for a further discussion).

\subsection{Non-convexity and Initialization}
Since the optimization problem is non-convex, the performance of the model also hinges on how the variational basis functions are initialized. We note that, when using the completely decoupled parametrization as in the experiments, we advocate to initialize the new basis for the mean and the covariance to be the same samples (e.g. from the current mini-batch) and then updating them separately online. This would encourage $\tilde{\mu}$ and $\tilde{\Sigma}$ to capture the information in a close subspace. 

Another idea to improve stability is to partially share the basis functions. For example, the first $M_\beta$ basis functions of the total $M_\alpha$ basis functions for the mean can be constrained to be the same as the basis functions for the covariance. Due to the redundancy in the mean parametrization, sharing part of the parameters can make the problem more well-conditioned and easier to optimize.

\subsection{Hybrid Subspace Parametrization} \label{app:hybrid subspace param}

An example\footnote{The idea of combining the partially shared representation and the canonical parametrization is brought up by Hugh Salimbeni in our discussion at the conference.} that combines the ideas from the above two sections is a hybrid subspace parametrization of  the variational Gaussian measure:
\begin{align}
\tilde{\mu} = \Psi_{\tilde{X}} \mbf{K}_{\tilde{X}}^{-1} \tilde{\mbf{m}} + \Psi_{\alpha_r} \mbf{a}_r   \qquad 
\tilde{\Sigma} = I + \Psi_{\tilde{X}} \mbf{K}_{\tilde{X}}^{-1} \left( \tilde{\mbf{S}} -  \mbf{K}_{\tilde{X}}  \right) \mbf{K}_{\tilde{X}}^{-1} \Psi_{\tilde{X}}^T  \label{eq:subspace param hybrid}
\end{align}
where $\tilde{\mbf{m}} \in \R^M$ and $\tilde{\mbf{S}} \succeq 0  \in \R^{M \times M}$ are equivalent to the posterior statistics on $\tilde{X}$ used in the conventional shared representation,  $\alpha_r$ denotes the additional inducing points that model the residual error, and $\mbf{a}_r$ is the corresponding coefficient. This representation is equivalent to setting $\beta = \tilde{X}$ and $ \mbf{B}^{-1} = - (\mbf{K}_{\tilde{X}}  + \mbf{K}_{\tilde{X}}( \tilde{\mbf{S}} -  \mbf{K}_{\tilde{X}} )^{-1} \mbf{K}_{\tilde{X}} ) $. 
The hybrid subspace parametrization gives a predictive model in the form
\begin{align}
\hat{m}_{|\mbf{y}}^H(x) &=  \mbf{k}_{x, \tilde{X}}  \mbf{K}_{\tilde{X}}^{-1}\tilde{\mbf{m}} + \mbf{k}_{x, \alpha_r} \mbf{a}_r \label{eq:mean hybrid}  \\ 
\hat{k}_{|\mbf{y}}^H(x,x') &=  k_{x, x'}   +  \mbf{k}_{x, \tilde{X}} \mbf{K}_{\tilde{X}}^{-1} \left( \tilde{\mbf{S}} -  \mbf{K}_{\tilde{X}}  \right) \mbf{K}_{\tilde{X}}^{-1} \mbf{k}_{\tilde{X},x'} \label{eq:cov hybrid}
\end{align}
Compared with~\eqref{eq:vsgpr}, the only difference is the residual term $\mbf{k}_{x, \alpha_r} \mbf{a}_r$, which helps modeling more complex functions. 

Therefore, this new, hybrid formulation shows a more direct connection to the conventional parametrization used in the GP literature (e.g.~\citep{hensman2015scalable}). This can also been seen in its associated objective function. Substitute the hybrid parametrization into the terms in the  KL divergence between Gaussian measures and we have the following relations: 
\begin{align*}
\frac{1}{2} \mbf{a}^T \mbf{K}_{{\alpha}}  \mbf{a}  
= \frac{1}{2} \tilde{\mbf{m}}^T \mbf{K}_{\tilde{X}}^{-1} \tilde{\mbf{m}} + \tilde{\mbf{m}}^T \mbf{K}_{\tilde{X}}^{-1} \mbf{K}_{\tilde{X}, \alpha_r} \mbf{a}_r + \frac{1}{2} \mbf{a}_r^T \mbf{K}_{\alpha_r} \mbf{a}_r
\end{align*}
\begin{align*}
\frac{-1}{2}  \tr{ \mbf{K}_{\beta}(  \mbf{B}^{-1}  + \mbf{K}_{{\beta}})^{-1}} =  \frac{1}{2}  \tr{ \tilde{\mbf{S}}  \mbf{K}_{\tilde{X}}^{-1}} - \frac{\abs{\tilde{X}}}{2}
\end{align*}
\begin{align*}
\frac{1}{2} \log |\mbf{I}  + \mbf{K}_{{\beta}} \mbf{B}|  
= -\log | \mbf{K}_{\tilde{X}}^{-1} \tilde{\mbf{S}}  |.  
\end{align*}
Therefore, the KL divergence term for the hybrid parametrization can be written as 
\begin{align}
\KL{q}{p} &= \frac{1}{2} \tilde{\mbf{m}}^T \mbf{K}_{\tilde{X}}^{-1} \tilde{\mbf{m}} + \tilde{\mbf{m}}^T  \mbf{K}_{\tilde{X}}^{-1} \mbf{K}_{\tilde{X}, \alpha_r} \mbf{a}_r + \frac{1}{2} \mbf{a}_r^T \mbf{K}_{\alpha_r} \mbf{a}_r  \nonumber \\
&\quad -\log | \mbf{K}_{\tilde{X}}^{-1} \tilde{\mbf{S}}  | + \frac{1}{2}  \tr{ \tilde{\mbf{S}}  \mbf{K}_{\tilde{X}}^{-1}} - \frac{\abs{\tilde{X}}}{2}  \label{eq:KL hybrid}
\end{align}
This KL divergence terms is exactly as the one used by~\citet{hensman2015scalable}, when $\tilde{\mbf{a}}_r = 0$. That is, the hybrid parametrization is a strict generalization of the canonical parametrization. Note in here $\tilde{\mbf{S}}$ is initialized as $\tilde{\mbf{S}} = \tilde{\mbf{K}}_{\tilde{X}}= \mbf{L}\mbf{L}^T$ and then its Cholesky factor $\mbf{L}$ is optimized afterwards. 

From~\eqref{eq:mean hybrid}, \eqref{eq:cov hybrid}, and~\eqref{eq:KL hybrid}, we can see the linear complexity of the decoupled model is preserved. The stochastic gradient can  be computed in linear time by performing sampling of the residual inducing points $\alpha_r$.

\section{Variational Inference with Decoupled Gaussian Processes} \label{app:VIDGP}

Here we provide the details of the variational inference problem used to learn \ourModel{}s:
\begin{align}
\max_{q(f),\theta}  \LL_\theta(q(f)& = 	\max_{q(f),\theta}  \int  q(f) \log \frac{p_\theta(y|f)p(f)} {q(f)}  \der{f} =  \max_{q(f),\theta}\E_{q}[ \log  p_\theta(y|f) ] - \KL{q}{p},
\end{align}

\subsection{KL Divergence}
\subsubsection{Evaluation}

First, we show how to evaluate the KL-divergence. We do so by extending the KL-divergence between two finite-dimensional subspace-parametrized Gaussian measures to infinite dimensional space and show that it is well-defined.

Recall for two $d$-dimensional Gaussian distributions $q(f) = \NN(f|\mu, \Sigma)$ and $p = \NN(f|\bar{\mu}, \bar{\Sigma})$, the KL-divergence is given as 
\begin{proposition}
\begin{align*}
	\KL{q}{p} &:= \int  \log \frac{q(f)}{p(f)} \der \mu_q(f) 
	=\int q(f) \log \frac{q(f)}{p(f)} \der f \\
	&= \frac{1}{2}\left( \tr{ \bar\Sigma^{-1} {\Sigma}  } + ({\mu} - \bar\mu)^T \bar\Sigma^{-1}({\mu} - \bar\mu)  + \ln \frac{ |\bar\Sigma| }{ |\Sigma|  }   -d \right) 
\end{align*} 
\label{th:KL}
\end{proposition}

Now consider $q$ and $p$ are subspace parametrized as 
\begin{align}
\begin{matrix}
p(f) = \NN(f| \bar\mu, \bar\Sigma)  = \NN(f|  \Psi_{\bar\alpha} \bar{\mbf{a}} , (I + \Psi_{\bar\beta} \bar{\mbf{B}} \Psi_{\bar\beta}^T)^{-1}  ) \\
q(f) = \NN({f|\mu}, {\Sigma})= \NN(f|  \Psi_{\alpha} {\mbf{a}} , (I + \Psi_{\beta} {\mbf{B}} \Psi_{\beta}^T)^{-1}  ).
\end{matrix} \label{eq:subspace param general}
\end{align}

By Proposition~\ref{th:KL}, we derive the representation of KL-divergence which is applicable even when $d$ is infinite. Recall in the infinite dimensional case, $\mu$, $\Sigma$, $\bar{\mu}$, and $\bar{\Sigma}$ are objects in the RKHS $\HH$ (Cameron-Martin space).
\begin{theorem}
Assume $q$ and $p$ are two subspace parametrized Gaussian measures given as~\eqref{eq:subspace param general}. Regardless of the dimension of $\HH$, the following holds 
\begin{align}
\KL{q}{p}&= \frac{-1}{2}  \tr{\left( \mbf{K}_{\beta} + \mbf{K}_{{\beta}, \bar\beta } \bar{\mbf{B}} \mbf{K}_{\bar\beta, {\beta} } \right)  ( {\mbf{B}}^{-1}  + \mbf{K}_{{\beta}})^{-1}     } 
 + \frac{1}{2} \log |\mbf{I}  + \mbf{K}_{{\beta}} {\mbf{B}}| \nonumber \\
&\quad + \frac{1}{2} {\mbf{a}}^T \left( \mbf{K}_{{\alpha}} + \mbf{K}_{{\alpha},\bar\beta} \bar{\mbf{B}}  \mbf{K}_{\bar\beta, {\alpha}}  \right)  {\mbf{a}} 
 - {\mbf{a}}^T   \left(  \mbf{K}_{{\alpha}, \bar\alpha} +  \mbf{K}_{ {\alpha}, \bar\beta} \bar{\mbf{B}}  \mbf{K}_{\bar\beta, \bar\alpha}  \right)   \bar{\mbf{a}} + C
 \label{eq:KL general}
\end{align}
where 
\begin{align*}
C = \frac{1}{2}\left( \tr{\mbf{K}_{\bar{\beta}} \bar{\mbf{B}} } - \log |\mbf{I}  + \mbf{K}_{\bar\beta} \bar{\mbf{B}}|  +  \bar{\mbf{a}}^T \left( \mbf{K}_{{\bar\alpha}} + \mbf{K}_{{\bar\alpha},\bar\beta} \mbf{\bar{B}}  \mbf{K}_{\bar\beta,{\bar\alpha}}  \right)  \bar{\mbf{a}}  \right)
\end{align*}
In particular, if $p$ is normal (i.e. $p(f) =\NN(f|0,I)$), then 
\begin{align*}
\KL{q}{p} &= \frac{1}{2} \mbf{a}^T \mbf{K}_{{\alpha}}  \mbf{a} 
+ \frac{1}{2} \log |\mbf{I}  + \mbf{K}_{{\beta}} \mbf{B}|  + \frac{-1}{2}  \tr{ \mbf{K}_{\beta}(  \mbf{B}^{-1}  + \mbf{K}_{{\beta}})^{-1}}  
\end{align*}

\end{theorem}

\begin{proof} 
\mbox{}\\
To prove, we derive each term in~\eqref{eq:KL general} as follows. 

First, we derive $\tr{ \bar{\Sigma}^{-1} {\Sigma}} - d$. 	
Define $ \mbf{R} =( \mbf{B}^{-1}  + \mbf{K}_\beta)^{-1}  $. Then we can write
\begin{align}
\Sigma &= (I + \Psi_\beta \mbf{B} \Psi_\beta^T)^{-1} 
= I - \Psi_\beta ( \mbf{B}^{-1}  + \Psi_\beta^T \Psi_\beta)^{-1} \Psi_\beta^T 
= I - \Psi_\beta \mbf{R} \Psi_\beta^T.
\label{eq:Sigma}
\end{align}
Using~\eqref{eq:Sigma}, we can derive
\begin{align*}
\bar{\Sigma}^{-1} {\Sigma} &= (I + \Psi_{\bar{\beta}} \bar{\mbf{B}} \Psi_{\bar{\beta}}^T) \left( I - \Psi_{\beta} \mbf{R} \Psi_{\beta}^T  \right) 
= I + \Psi_{\bar{\beta}} \bar{\mbf{B}} \Psi_{\bar{\beta}}^T - \Psi_{\beta} \mbf{R} \Psi_{\beta}^T - \Psi_{\bar\beta} \mbf{\bar{B}} \mbf{K}_{\bar\beta, {\beta}} \mbf{R} \Psi_{\beta}^T 
\end{align*}
and therefore
\begin{align*}
\tr{\bar{\Sigma}^{-1} {\Sigma}} -d &= \tr{I}- d + \tr{\mbf{K}_{\bar{\beta}} \bar{\mbf{B}} } 
- \tr{\mbf{R} \left( \mbf{K}_{\beta} + \mbf{K}_{{\beta}, \bar\beta } \bar{\mbf{B}}  \mbf{K}_{\bar\beta, {\beta} } \right) }\\
&= \tr{\mbf{K}_{\bar{\beta}} \bar{\mbf{B}} } 
- \tr{\mbf{R} \left( \mbf{K}_{\beta} + \mbf{K}_{{\beta}, \bar\beta } \bar{\mbf{B}}  \mbf{K}_{\bar\beta, {\beta} } \right) }
\end{align*}
Note this term does not depend on the ambient dimension.

Second, we derive $\log  ( |\bar\Sigma| /|{\Sigma}|) $:
Since 
\begin{align*}
\log |{\Sigma}^{-1}| 
= \log  |\mbf{B}^{-1}  + \mbf{K}_{\beta} | | \mbf{B}|   
=  \log |\mbf{I}  + \mbf{K}_{\beta} \mbf{B}|.
\end{align*}
it holds that 
\begin{align*}
\log \frac{  |\bar\Sigma| }{ |{\Sigma}|} =  \log |\mbf{I}  + \mbf{K}_{\beta} \mbf{B}| - \log |\mbf{I}  + \mbf{K}_{\bar\beta} \bar{\mbf{B}}|.
\end{align*}

Finally, we  derive the quadratic term:
\begin{align*}
&({\mu} - \bar\mu)^T \bar\Sigma^{-1}({\mu} -\bar\mu) \\
&= {\mu}^T \bar\Sigma^{-1} {\mu} - 2 \bar{\mu}^T \bar\Sigma^{-1} \mu 
 +  \bar\mu \bar\Sigma^{-1} \bar\mu 
\\
&= {\mbf{a} }^T  \Psi_{\alpha}^T \left(  I + \Psi_{\bar\beta} \mbf{ \bar{B} }\Psi_{\bar{\beta}}^T  \right)  \Psi_{\alpha} {\mbf{a}} 
-  2{\mbf{\bar{a}}}^T  \Psi_{\bar\alpha}^T \left(  I + \Psi_{\bar\beta} \bar{\mbf{B}}\Psi_{\bar\beta}^T  \right)  \Psi_{\alpha} {\mbf{a}} 
+ {\bar{\mbf{a}} }^T  \Psi_{\bar\alpha}^T \left(  I + \Psi_{\bar\beta} \mbf{ \bar{B} }\Psi_{\bar{\beta}}^T  \right)  \Psi_{\bar\alpha} \bar5{\mbf{a}} 
\\
&= {\mbf{a}}^T \left( \mbf{K}_{{\alpha}} + \mbf{K}_{{\alpha},\bar\beta} \mbf{\bar{B}}  \mbf{K}_{\bar\beta,{\alpha}}  \right)  {\mbf{a}} 
- 2{\mbf{\bar{a}}}^T   \left(  \mbf{K}_{\bar\alpha, {\alpha}} +  \mbf{K}_{\bar\alpha,\bar\beta} \mbf{\bar{B}}\mbf{K}_{\bar\beta, {\alpha}}  \right)   {\mbf{a}}
 +  \bar{\mbf{a}}^T \left( \mbf{K}_{{\bar\alpha}} + \mbf{K}_{{\bar\alpha},\bar\beta} \mbf{\bar{B}}  \mbf{K}_{\bar\beta,{\bar\alpha}}  \right)  \bar{\mbf{a}} 
\end{align*}
\end{proof}

\paragraph{Remarks}
 The above expression is well defined even when $\mbf{B}\succeq 0$, because $( {\mbf{B}}^{-1}  + \mbf{K}_{{\beta}})^{-1} =\mbf{B} ( \mbf{I}  + \mbf{K}_{{\beta}} \mbf{B})^{-1}$. Particularly,  we can parametrize  $\mbf{B} = \mbf{L}\mbf{L}^T $ with Cholesky factor $\mbf{L}  \in \R^{M_\beta \times M_\beta} $ in practice so the problem is unconstrained. The required terms can be stably computed: $ \left( \mbf{B}^{-1} + \mbf{K}_{\beta}  \right)^{-1} = \mbf{L} \mbf{H}^{-1} \mbf{L}^T  $ and $\log |\mbf{I}  + \mbf{K}_{{\beta}} \mbf{B}| = \log |\mbf{H}| $, where  $\mbf{H} = \mbf{I} + \mbf{L}^T \mbf{K}_{\beta}\mbf{L} $.

\subsubsection{Gradients}
Here we derive the equations of the gradient of the variational inference problem of \ourMethod. The purpose here is to show the complexity of calculating the gradients. These equations are useful in implementing \ourMethod~using basic linear algebra routines, while computational-graph libraries based on automatic differentiation are also applicable and easier to apply.

To derive the gradients, we first introduce some short-hand
\begin{align*}
\mbf{G}_{\alpha} &=  \mbf{K}_{{\alpha}} + \mbf{K}_{{\alpha},\bar\beta} \bar{\mbf{B}} \mbf{K}_{\bar\beta, {\alpha} } \\
\mbf{G}_{\alpha, \bar{\alpha}} &= \mbf{K}_{{\alpha}, \bar\alpha} +  \mbf{K}_{ {\alpha}, \bar\beta} \bar{\mbf{B}}  \mbf{K}_{\bar\beta, \bar\alpha}   \\
\mbf{G}_{\beta} &= \mbf{K}_{\beta} + \mbf{K}_{{\beta}, \bar\beta } \bar{\mbf{B}} \mbf{K}_{\bar\beta, {\beta} }
\end{align*}
and write $\KL{q}{p}$ as
\begin{align*}
\KL{q}{p}= \frac{-1}{2}  \tr{ \mbf{G}_{\beta}( {\mbf{B}}^{-1}  + \mbf{K}_{{\beta}})^{-1}     } 
+ \frac{1}{2} \log | \mbf{I}  + \mbf{K}_{{\beta}} {\mbf{B}}|
+ \frac{1}{2} {\mbf{a}}^T \mbf{G}_{\alpha}  {\mbf{a}} 
- {\mbf{a}}^T   \mbf{G}_{\alpha, \bar{\alpha}}   \bar{\mbf{a}}.
\end{align*}

We then give the equations to compute the derivatives below.
For compactness of notation, we use $\odot$ to denote element-wise product and  use $\mbf{1}$ to denote the vector of ones. In addition, we  introduce a linear operator $\diag$  with overloaded definitions:
\begin{enumerate}
\item $\diag: \R^N \to \R^{N\times N}$ which constructs a diagonal matrix from a vector 
\item  $\diag: \R^{N\times N} \to \R^N$ which extracts the diagonal elements of a matrix to a vector.
\end{enumerate}

\begin{proposition} \label{th:KL derivative}
The gradients of $\KL{q}{p}$ is as follows:
\begin{align*}
\nabla_{\mbf{a}}\KL{q}{p} &=  \mbf{G}_{\alpha} \mbf{a} - \mbf{G}_{\alpha,\bar{\alpha}}   \bar{\mbf{a}}\\
\nabla_{\alpha}\KL{q}{p} &= 
\diag(\mbf{a}) \left( \partial_{\alpha} \mbf{G}_{\alpha} \mbf{a} -\partial_{\alpha}  \mbf{G}_{\alpha,\bar{\alpha}}   \bar{\mbf{a}}   \right) \\
\nabla_{\mbf{B}}\KL{q}{p}
&=  \frac{1}{2} ( \mbf{I}  + \mbf{K}_{{\beta}} \mbf{B})^{-1} \left( \mbf{K}_{{\beta}} \mbf{B} \mbf{K}_{{\beta}} - \mbf{\Delta}_\beta  \right) ( \mbf{I}  + \mbf{B} \mbf{K}_{{\beta}} )^{-1}\\
\nabla_{\beta}\KL{q}{p} 
&= \left( \partial_{\beta} \mbf{K}_\beta \odot  ( {\mbf{B}}^{-1}  + \mbf{K}_{{\beta}})^{-1} \mbf{G}_{\beta} ( {\mbf{B}}^{-1}  + \mbf{K}_{{\beta}})^{-1} \right) \mbf{1}
- \left(  \partial_{\beta} \mbf{\Delta}_\beta \odot ( \mbf{B}^{-1}  + \mbf{K}_{{\beta}} )^{-1}   \right) \mbf{1}
\end{align*}
where $\mbf{\Delta}_\beta = \mbf{G}_\beta - \mbf{K}_{\beta}$ and $\partial$ is  defined as the partial derivative with respect to the left argument.\footnote{The additional factor of 2 is due to $\mbf{K}_\beta$ is symmetric.}  In particular, if the $p$ is normal,
\begin{align*}
\nabla_{\mbf{a}}\KL{q}{p} &=  \mbf{K}_{\alpha} \mbf{a} \\
\nabla_{\alpha}\KL{q}{p} &= 
\diag(\mbf{a})  \partial_{\alpha} \mbf{K}_{\alpha} \mbf{a} \\
\nabla_{\mbf{B}}\KL{q}{p}
&=  \frac{1}{2} ( \mbf{I}  + \mbf{K}_{{\beta}} \mbf{B})^{-1}  \mbf{K}_{{\beta}} \mbf{B} \mbf{K}_{{\beta}}  ( \mbf{I}  + \mbf{B} \mbf{K}_{{\beta}} )^{-1}\\
\nabla_{\beta}\KL{q}{p} 
&= \left( \partial_{\beta} \mbf{K}_\beta \odot  ( {\mbf{B}}^{-1}  + \mbf{K}_{{\beta}})^{-1} \mbf{K}_{\beta} ( {\mbf{B}}^{-1}  + \mbf{K}_{{\beta}})^{-1} \right) \mbf{1}
\end{align*}
\end{proposition}
The derivation of Proposition~\ref{th:KL derivative} is simply mechanical, so we omit it here.

Here we only show the derivative with respect to $\mbf{B}$. 
Suppose $\mbf{B} = \mbf{L}\mbf{L}^T$. Then one can apply the chain rule and get 
\begin{align*}
\nabla_{\mbf{L}}\KL{q}{p} =    2 \nabla_{B}\KL{q}{p} \mbf{L}.
\end{align*}

\subsection{Expected Log-Likelihood}

\subsubsection{Evaluation}
The evaluation of the expected log-likelihood depends on the mean and covariance in~\eqref{eq:decoupled GP} , which we repeat here
\begin{align*}
\hat{m}_{|\mbf{y}}^{\alpha}(x) =  \mbf{k}_{x, \alpha} \bm{a}, \qquad \hat{k}_{|\mbf{y}}^{\beta}(x,x') =  k_{x, x'}   -  \mbf{k}_{x,\beta}  \left( \mbf{B}^{-1} + \mbf{K}_{\beta}  \right)^{-1}  \mbf{k}_{\beta,x'}.
\end{align*}
Its derivation is trivial by the definition of $q$ in~\eqref{eq:subspace param general} and~\eqref{eq:Sigma}. For $N$ observations, the vector form $\hat{\mbf{m}}\in \R^N$  and $\hat{\mbf{s}} \in \R^N$ of the mean and the covariance above evaluated on each  observation can be computed in $O(N)$ as
\begin{align*}
	\hat{\mbf{m}} &= \mbf{K}_{X,\alpha} \mbf{a}    \\
	\hat{\mbf{s}} &= \diag\left( \mbf{K}_X - \mbf{K}_{X,\beta}(\mbf{B}^{-1} + \mbf{K}_{\beta} )^{-1} \mbf{K}_{\beta,X} \right) \\
	&= \diag(\mbf{K}_X) - \left( \mbf{K}_{X,\beta}\odot(  \mbf{K}_{X,\beta} (\mbf{B}^{-1} + \mbf{K}_{\beta} )^{-1}) \right)\mbf{1} \\
	&= \diag(\mbf{K}_X) - \left( \mbf{K}_{X,\beta}\odot(  \mbf{K}_{X,\beta} \mbf{B}( \mbf{I} + \mbf{K}_{\beta}\mbf{B} )^{-1}) \right)\mbf{1}.
\end{align*}
Given $\hat{\mbf{m}}$ and $\hat{\mbf{s}}$, the expected log-likelihood can be evaluated either in closed-form  for  Gaussian likelihood or by sampling for general likelihoods.

\subsubsection{Gradients}
The computation of the gradients of the expected log-likelihood can be completed in two steps. First, we compute the gradients of $\E_{q}[ \log  p_\theta(y|f) ]$ with respect to $(\theta, \hat{\mbf{m}}, \hat{\mbf{s}})$ (i.e. $\nabla_{\hat{\mbf{m}}}e$,  $\nabla_{\hat{\mbf{s}}}e$,  and $\nabla_{\hat{\mbf{\theta}}}e$ ). 
Because $\log  p_\theta(y|f)$ is the sum of $N$ terms, this step can be done in $O(N)$: for each observation $x$, let $q(f(x)) = \NN(f(x)|\hat{m},\hat{s})$ be a scalar Gaussian; under standard regularity conditions, we have
\begin{align*}
	\nabla_{\hat{m}} \E_{q} [ \log p_\theta(y|f(x))  ] &= \E_{q} [ \nabla_{\hat{m}} \log q(f(x)) \log p_\theta(y|f(x))  ] \\
	\nabla_{\hat{s}} \E_{q} [ \log p_\theta(y|f(x))  ] &= \E_{q} [ \nabla_{\hat{s}} \log q(f(x)) \log p_\theta(y|f(x))  ] \\
	\nabla_{\theta} \E_{q} [ \log p_\theta(y|f(x))  ] &= \E_{q} [ \nabla_{\theta}  \log p_\theta(y|f(x))  ]	
\end{align*}
where $\nabla_{\theta}  \log p_\theta(y|f(x)) $ can be found, for example, in~\cite{sheth2015sparse}. The above can be calculated in closed-form  for  Gaussian likelihood or by sampling for general likelihoods.

 Next we propagate these gradients by chain rule. The results are summarized below. 
\begin{proposition} \label{th:expected log-likeli derivative}
Let $e = \E_{q}[ \log  p_\theta(y|f) ]$. Suppose $k(x,x') = \rho^2 g_s(x,x')$ for some hyper-parameters $\rho,s\in\R$.
The gradients of e are as follows: 
\begin{align*}
\nabla_a e  &= \mbf{K}_{X, \alpha}^T \nabla_{\hat{\mbf{m}} } e\\
\nabla_{\alpha} e  &= \diag(\mbf{a})\partial \mbf{K}_{X, \alpha}^T \nabla_{\hat{\mbf{m}} } e \\
\nabla_{\mbf{B}} e
&= - (\mbf{I} +  \mbf{K}_{\beta} \mbf{B} )^{-1}  \mbf{K}_{X,\beta}^T \diag( \nabla_{\hat{\mbf{s}}} e)  \mbf{K}_{X,\beta} (\mbf{I} + \mbf{B} \mbf{K}_{\beta}  )^{-1}
\\
\nabla_{\beta} \hat{e} 
&= 2 (\partial \mbf{K}_{\beta}^T \odot ( \mbf{\Omega}  \diag(\nabla_{\hat{\mbf{s}}} e)\mbf{\Omega}^T  ) )\bm{1}  - 2  (  \mbf{\Omega} \odot \partial \mbf{K}_{\beta, X})\nabla_{\hat{\mbf{s}}} e\\
\nabla_{\log\rho} e &= \hat{\mbf{m}}^T  \nabla_{\hat{\mbf{m}} } e+ 2 \hat{\mbf{s}}^T \nabla_{\hat{{\mbf{s}}}} e \\
\nabla_{s} e &= (\partial_s \mbf{K}_{X,\alpha} \mbf{a} )^T \nabla_{\hat{\mbf{m}} } e 
-  2\bm{1}^T \left( \mbf{\Omega} \odot  \partial_s \mbf{K}_{\beta, X}    \right) \nabla_{\hat{{\mbf{s}}}} e
\end{align*}
where $\mbf{\Omega} = \mbf{B} ( \mbf{I} +  \mbf{K}_{\beta} \mbf{B} )^{-1}  \mbf{K}_{\beta,X}$.
\end{proposition}
The derivation of Proposition~\ref{th:expected log-likeli derivative} is only technical, so we omit it here.

\section{Experiment Setup} \label{app:exp setup}
\subsection{The Covariance Function}\vspace{-1mm}
For all the models, we assume the prior is zero mean and has covariance defined by a \textsc{se-ard} kernel~\cite{rasmussen2006gaussian} 
\[
 k(x,x') = \rho^2 \phi_{x}^T \phi_{x} = \rho^2 \prod_{d=1}^D \exp( \frac{ -(x_d - x'_d)^2}{ 2 s_d^2} ),\]
where $s_d>0$ is the length scale of dimension $d$. For the variational posterior, we use the generalized \textsc{se-ard} kernel~\cite{cheng2016incremental}
\begin{align}
  \psi_x^T \psi_{x'} &= \prod_{d=1}^{D} \left(\frac{ 2 l_{x,d} l_{x',d}  }{  l_{x,d}^2 + l_{x',d}^2} \right)^{1/2} \exp\left( -  \frac{\norm{x_d-x_d'}^2}{ l_{x,d}^2 + l_{x',d}^2}   \right)\hspace{-1mm},
	\hspace{-1mm}  \label{eq:ard-gse kernel}
\end{align}
where $l_{x,d} = s_d \cdot c_{x,d}$ is the length-scale parameter.  That is, we evaluate 
\[
	\C[ L_m f(\tilde{x}_m),  L_n f(\tilde{x}_n)] = \psi_{\tilde{x}_m}^T \psi_{\tilde{x}_n}
\]
where the associated length-scalar parameters  implicitly define the linear operators $L_m$ and $L_n$.

This kernel is first introduced in~\cite{walder2008sparse} by convoluting a \textsc{se-ard} kernel with Gaussian integral kernels, 
and later modified into its current form~\eqref{eq:ard-gse kernel} in~\cite{cheng2016incremental}. From~\eqref{eq:ard-gse kernel}, we see it contains \textsc{se-ard} as a special case. That is, $\psi_x = \phi_x$ when $c_{x,d} = 1$, $\forall d \in \{1,\dots,D\}$. But in general $c_{x,d}$ can be a function of $x$. 
Therefore, it can be shown that $\psi_x$ spans an RKHS that contains the RKHSs spanned by $\phi_x$ for all length-scales, and every cross covariance can be computed as $\C[ L_m f(\tilde{x}_m), f(x)  ] = \rho \psi_{\tilde{x}_m}^T \phi_{x}  $. 

Note: all the algorithms in our comparisons use this generalized \textsc{se-ard} kernel.

\subsection{Online Learning Procedure}

Algorithm~\ref{alg:algo} summarizes the online learning procedure used by all stochastic algorithms (the algorithms differs only in whether the bases are shared and how the model is updated; see Table~\ref{tb:algorithms}.), where each learner has to optimize all the parameters on-the-fly using \emph{i.i.d.} data. The hyper-parameters are first initialized heuristically by median trick using the first mini-batch (in the GPR experiments,  $s_d$ is initialized as the median of pairwise distances of the sampled observations; $\sigma^2$ is initialized as the variance of the sampled outputs; $\rho =1$). We incrementally build up the variational posterior by including $N_{\Delta}\leq N_m$ observations in each mini-batch as the initialization of new variational basis functions (we initialize a new variational basis as $\tilde{x}_m = x_n$ and $c_{\tilde{x},d} = 1$, where $x_n$ is a sample from the current mini-batch). Then all the hyper-parameters and the variational parameters are updated online. These steps are repeated for $T$ iterations.

\clearpage
\section{Complete Experimental Results} \label{app:exp results}
\subsection{Experimental Results on KUKA datasets}
\begin{table}[!tbh]
\centering
{\scriptsize
\subfloat[Variational Lower Bound ($10^5$)]{
  \begin{tabular}{lllllll}
      \toprule
         & \ourMethod{}  & \textsc{svi}   &  i\textsc{vsgpr} & \textsc{vsgpr} & \textsc{gpr} \\
      \midrule
	  $Y_1$ & \textbf{0.985} & 0.336 & 0.411 & 0.085 & -3.840\\   
	  $Y_2$ & \textbf{1.359} & 0.458 & 0.799 & 0.468 & -23.218\\   
	  $Y_3$ & \textbf{0.951} & 0.312 & 0.543 & 0.158 & -8.145\\   
	  $Y_4$ & \textbf{1.453} & 0.528 & 0.906 & 0.722 & -0.965\\   
	  $Y_5$ & \textbf{1.350} & 0.311 & 0.377 & 0.425 & -0.990\\   
	  $Y_6$ & \textbf{1.278} & 0.367 & 0.631 & 0.559 & -0.639\\   
	  $Y_7$ & \textbf{1.458} & 0.425 & 0.877 & 0.886 & 0.449\\   
	  \midrule 
	  mean & \textbf{1.262} & 0.391 & 0.649 & 0.472 & -5.335\\   
	  std & \textbf{0.195} & 0.076 & 0.201 & 0.265 & 7.777\\   
      \bottomrule
    \end{tabular}} 

\subfloat[Prediction Error (nMSE)]{
  \begin{tabular}{lllllll}
      \toprule
         & \ourMethod{}  & \textsc{svi}   &  i\textsc{vsgpr} & \textsc{vsgpr} & \textsc{gpr} \\
      \midrule
     $Y_1$ & \textbf{0.058} & 0.186 & 0.165 & 0.171 & 0.257\\    
     $Y_2$ & \textbf{0.028} & 0.146 & 0.095 & 0.126 & 0.249\\    
     $Y_3$ & \textbf{0.058} & 0.195 & 0.133 & 0.181 & 0.298\\    
     $Y_4$ & \textbf{0.027} & 0.124 & 0.088 & 0.114 & 0.198\\    
     $Y_5$ & \textbf{0.028} & 0.195 & 0.178 & 0.132 & 0.243\\    
     $Y_6$ & \textbf{0.034} & 0.178 & 0.137 & 0.140 & 0.224\\    
     $Y_7$ & \textbf{0.028} & 0.155 & 0.099 & 0.108 & 0.146\\    
     \midrule 
     mean & \textbf{0.037} & 0.169 & 0.128 & 0.139 & 0.231\\    
     std & \textbf{0.013} & 0.025 & 0.033 & 0.026 & 0.045\\       
      \bottomrule
    \end{tabular}} 
}
  \caption{Experimental results of \textsc{kuka}$_1$ after 2,000 iteration.
  $Y_i$ denotes the $i$th output.}
\end{table} 
\begin{table}[!tbh]
\centering
{\scriptsize
\subfloat[Variational Lower Bound ($10^5$)]{
  \begin{tabular}{lllllll}
      \toprule
         & \ourMethod{}  & \textsc{svi}   &  i\textsc{vsgpr} & \textsc{vsgpr} & \textsc{gpr} \\
      \midrule
     	$Y_1$ & \textbf{1.047} & 0.398 & 0.631 & 0.399 & -3.709\\   
     	$Y_2$ & \textbf{1.387} & 0.450 & 0.767 & 0.515 & -31.315\\   
     	$Y_3$ & \textbf{0.976} & 0.321 & 0.568 & 0.232 & -12.230\\   
     	$Y_4$ & \textbf{1.404} & 0.507 & 0.630 & 0.654 & -1.026\\   
     	$Y_5$ & \textbf{1.332} & 0.317 & 0.378 & 0.511 & -0.340\\   
     	$Y_6$ & \textbf{1.260} & 0.368 & 0.585 & 0.538 & -0.221\\   
     	$Y_7$ & \textbf{1.405} & 0.437 & 0.519 & 0.918 & 0.526\\   
     	\midrule 
     	mean & \textbf{1.259} & 0.400 & 0.583 & 0.538 & -6.902\\   
     	std & \textbf{0.165} & 0.065 & 0.110 & 0.197 & 10.770\\   
      \bottomrule
    \end{tabular}} 
\\
\subfloat[Prediction Error (nMSE)]{
  \begin{tabular}{lllllll}
      \toprule
         & \ourMethod{}  & \textsc{svi}   &  i\textsc{vsgpr} & \textsc{vsgpr} & \textsc{gpr} \\
      \midrule
     $Y_1$ & \textbf{0.056} & 0.168 & 0.126 & 0.151 & 0.281\\    
     $Y_2$ & \textbf{0.026} & 0.147 & 0.102 & 0.124 & 0.248\\    
     $Y_3$ & \textbf{0.056} & 0.194 & 0.127 & 0.179 & 0.325\\    
     $Y_4$ & \textbf{0.029} & 0.127 & 0.127 & 0.110 & 0.186\\    
     $Y_5$ & \textbf{0.029} & 0.189 & 0.170 & 0.125 & 0.232\\    
     $Y_6$ & \textbf{0.035} & 0.181 & 0.144 & 0.144 & 0.232\\    
     $Y_7$ & \textbf{0.034} & 0.152 & 0.166 & 0.104 & 0.133\\    
     \midrule 
      mean & \textbf{0.038} & 0.166 & 0.137 & 0.134 & 0.234\\    
     std & \textbf{0.012} & 0.023 & 0.022 & 0.024 & 0.058\\       
      \bottomrule
    \end{tabular}} 
}
  \caption{Experimental results of \textsc{kuka}$_2$ after 2,000 iterations. $Y_i$ denotes the $i$th output.}
\end{table} 

\clearpage
\subsection{Experimental Results on MuJoCo datasets}
\begin{table}[htb!]
\centering
{\scriptsize
\subfloat[Variational Lower Bound ($10^5$)]{
  \begin{tabular}{lllllll}
      \toprule
         & \ourMethod{}  & \textsc{svi}   &  i\textsc{vsgpr} & \textsc{vsgpr} & \textsc{gpr} \\
      \midrule
     	$Y_1$ & \textbf{7.373} & 3.195 & 5.948 & 4.312 & -22.256\\   
     	$Y_2$ & \textbf{6.019} & 2.141 & 3.905 & 2.328 & -45.351\\   
     	$Y_3$ & \textbf{6.350} & 2.543 & 4.695 & 2.991 & -147.881\\   
     	$Y_4$ & \textbf{5.852} & 2.417 & 4.792 & 2.468 & -23.999\\   
     	$Y_5$ & \textbf{6.280} & 2.609 & 5.316 & 3.622 & -8.626\\   
     	$Y_6$ & \textbf{5.152} & 1.043 & 4.418 & 3.452 & -11296.669\\   
     	$Y_7$ & \textbf{5.270} & 2.093 & 4.183 & 1.676 & -7745.055\\   
     	$Y_8$ & \textbf{6.471} & 2.585 & 5.040 & 3.068 & -47.540\\   
     	$Y_9$ & \textbf{5.293} & 0.979 & 2.592 & 1.482 & -73477.168\\   
     	\midrule 
     	mean & \textbf{6.007} & 2.178 & 4.543 & 2.822 & -10312.727\\   
     	std & \textbf{0.673} & 0.692 & 0.898 & 0.871 & 22679.778\\   
      \bottomrule
    \end{tabular}} 
\\
\subfloat[Prediction Error (nMSE)]{
  \begin{tabular}{lllllll}
      \toprule
         & \ourMethod{}  & \textsc{svi}   &  i\textsc{vsgpr} & \textsc{vsgpr} & \textsc{gpr} \\
      \midrule
		$Y_1$ & \textbf{0.049} & 0.087 & 0.067 & 0.088 & 0.133\\    
		$Y_2$ & \textbf{0.068} & 0.163 & 0.112 & 0.122 & 0.196\\    
		$Y_3$ & \textbf{0.064} & 0.134 & 0.091 & 0.112 & 0.213\\    
		$Y_4$ & \textbf{0.073} & 0.144 & 0.087 & 0.121 & 0.179\\    
		$Y_5$ & \textbf{0.068} & 0.132 & 0.080 & 0.103 & 0.159\\    
		$Y_6$ & \textbf{0.094} & 0.251 & 0.107 & 0.131 & 0.253\\    
		$Y_7$ & \textbf{0.084} & 0.168 & 0.100 & 0.145 & 0.348\\    
		$Y_8$ & \textbf{0.063} & 0.132 & 0.087 & 0.104 & 0.178\\    
		$Y_9$ & \textbf{0.088} & 0.255 & 0.165 & 0.131 & 0.258\\    
		\midrule 
		 mean & \textbf{0.072} & 0.163 & 0.099 & 0.118 & 0.213\\    
		std & \textbf{0.013} & 0.053 & 0.026 & 0.016 & 0.061\\     
      \bottomrule
    \end{tabular}} 
}
  \caption{Experimental results of \textsc{mujoco}$_1$ after 2,000 iterations. $Y_i$ denotes the $i$th output.}
\end{table} 
\begin{table}[htb!]
\centering
{\scriptsize
\subfloat[Variational Lower Bound ($10^5$)]{
  \begin{tabular}{lllllll}
      \toprule
         & \ourMethod{}  & \textsc{svi}   &  i\textsc{vsgpr} & \textsc{vsgpr} & \textsc{gpr} \\
      \midrule
     	$Y_1$ & \textbf{7.249} & 3.013 & 6.429 & 4.161 & -33.219\\   
     	$Y_2$ & \textbf{5.994} & 2.475 & 4.800 & 2.770 & -23.276\\   
     	$Y_3$ & \textbf{6.239} & 2.258 & 4.819 & 3.044 & -59.757\\   
     	$Y_4$ & \textbf{5.935} & 2.093 & 4.489 & 2.547 & -27.259\\   
     	$Y_5$ & \textbf{6.387} & 2.452 & 5.457 & 3.725 & -1.786\\   
     	$Y_6$ & \textbf{7.320} & 1.087 & 4.639 & 4.043 & -24.198\\   
     	$Y_7$ & \textbf{5.346} & 1.754 & 3.947 & 1.667 & -255179.052\\   
     	$Y_8$ & \textbf{6.448} & 2.505 & 5.193 & 3.812 & -190.294\\   
     	$Y_9$ & \textbf{6.237} & 0.683 & 2.596 & 2.241 & -37673.328\\   
     	\midrule 
     	mean & \textbf{6.350} & 2.036 & 4.708 & 3.112 & -32579.130\\   
     	std & \textbf{0.586} & 0.699 & 0.993 & 0.825 & 79570.425\\   
      \bottomrule
    \end{tabular}} 
\\
\subfloat[Prediction Error (nMSE)]{
  \begin{tabular}{lllllll}
      \toprule
         & \ourMethod{}  & \textsc{svi}   &  i\textsc{vsgpr} & \textsc{vsgpr} & \textsc{gpr} \\
      \midrule
		$Y_1$ & \textbf{0.051} & 0.095 & 0.056 & 0.085 & 0.138\\    
		$Y_2$ & \textbf{0.069} & 0.133 & 0.085 & 0.111 & 0.186\\    
		$Y_3$ & \textbf{0.066} & 0.149 & 0.087 & 0.113 & 0.182\\    
		$Y_4$ & \textbf{0.071} & 0.160 & 0.094 & 0.127 & 0.197\\    
		$Y_5$ & \textbf{0.065} & 0.137 & 0.074 & 0.101 & 0.148\\    
		$Y_6$ & \textbf{0.051} & 0.241 & 0.097 & 0.073 & 0.139\\    
		$Y_7$ & \textbf{0.081} & 0.187 & 0.107 & 0.142 & 0.363\\    
		$Y_8$ & \textbf{0.063} & 0.133 & 0.081 & 0.106 & 0.214\\    
		$Y_9$ & \textbf{0.067} & 0.270 & 0.157 & 0.106 & 0.300\\    
		\midrule \ifCompileAppendix
		 mean & \textbf{0.065} & 0.167 & 0.093 & 0.107 & 0.207\\    
		std & \textbf{0.009} & 0.053 & 0.026 & 0.019 & 0.072\\           
      \bottomrule
    \end{tabular}} 
}
  \caption{Experimental results of \textsc{mujoco}$_2$ after 2,000 iterations. $Y_i$ denotes the $i$th output.}
\end{table} 

\fi

\end{document}